\documentclass[twoside,11pt]{article}

\usepackage{blindtext}
\usepackage{booktabs}

\usepackage{amssymb}
\usepackage{mathtools}
\usepackage{amsthm}
\usepackage{microtype}
\usepackage{graphicx}
\usepackage{subfigure}
\usepackage{booktabs} 
\usepackage{amsfonts} 
\usepackage{amsmath} 
\usepackage{bbm, dsfont}
\usepackage{enumerate}
\usepackage{url} 
\usepackage{xcolor}
\usepackage{tikz}
\usepackage{units}
\usepackage{multicol}
\usepackage{algorithm2e}

\newcommand\x{\mathbf{x}}
\newcommand{\y}{\mathbf{y}}
\newcommand{\X}{\mathbf{X}}

\newcommand{\bb}{\boldsymbol{\beta}}
\newcommand{\tb}{\tilde{\boldsymbol{\beta}}}
\newcommand{\hb}{\hat{\boldsymbol{\beta}}}
\newcommand{\uu}{\mathbf{u}}
\newcommand{\bl}{\boldsymbol{\lambda}}

\renewcommand{\b}{\boldsymbol\beta}

\newcommand{\sechead}[1]{}

\usepackage[preprint]{jmlr2e}

\newtheorem{innercustomgeneric}{\customgenericname}
\providecommand{\customgenericname}{}
\newcommand{\newcustomtheorem}[2]{%
  \newenvironment{#1}[1]
  {%
   \renewcommand\customgenericname{#2}%
   \renewcommand\theinnercustomgeneric{##1}%
   \innercustomgeneric
  }
  {\endinnercustomgeneric}
}

\newcustomtheorem{customthm}{Theorem}
\newcustomtheorem{customlemma}{Lemma}

\usepackage{lastpage}
\jmlrheading{23}{2022}{1-\pageref{LastPage}}{1/21; Revised 5/22}{9/22}{21-0000}{Nathan Wycoff et al}

\ShortHeadings{Proximal Iteration for Nonlinear Adaptive Lasso}{Wycoff, Singh, Arab and Donato}
\firstpageno{1}

\begin{document}

\title{Proximal Iteration for Nonlinear Adaptive Lasso}


\author{\name  Nathan Wycoff \email nwycoff@umass.edu \\
       \addr Department of Mathematics and Statistics\\
       University of Massachusetts, 
       Amherst, MA 01003, USA
       \AND
       \name  Lisa O. Singh \email lisa.singh@georgetown.edu \\
       \addr Computer Science Department and McCourt School of Public Policy\\
       Georgetown University,
       Washington, DC 20007, USA
       \AND
       \name  Ali Arab \email ali.arab@georgetown.edu \\
       \addr Mathematics and Statistics Department\\
       Georgetown University,
       Washington, DC 20007, USA
       \AND
       \name Katharine M. Donato \email katharine.donato@georgetown.edu \\
       \addr School of Foreign Service\\
       Georgetown University,\\
       Washington, DC 20007, USA}

\editor{}

\maketitle

\begin{abstract}

Augmenting a smooth cost function with an $\ell_1$ penalty allows analysts to efficiently conduct estimation and variable selection simultaneously in sophisticated models and can be efficiently implemented using proximal gradient methods.
However, one drawback of the $\ell_1$ penalty is bias: nonzero parameters are underestimated in magnitude, motivating techniques such as the Adaptive Lasso which endow each parameter with its own penalty coefficient. 
But it's not clear how these parameter-specific penalties should be set in complex models.
In this article, we study the approach of treating the penalty coefficients as additional decision variables to be learned in a \textit{Maximum a Posteriori} manner, developing a proximal gradient approach to joint optimization of these together with the parameters of any differentiable cost function.
Beyond reducing bias in estimates, this procedure can also encourage arbitrary sparsity structure via a prior on the penalty coefficients.
We compare our method to implementations of specific sparsity structures for non-Gaussian regression on synthetic and real datasets, finding our more general method to be competitive in terms of both speed and accuracy.
We then consider nonlinear models for two case studies: COVID-19 vaccination behavior and international refugee movement, highlighting the applicability of this approach to complex problems and intricate sparsity structures. 

\end{abstract}


\renewcommand{\b}{\boldsymbol{\beta}}
\renewcommand{\l}{\boldsymbol{\lambda}}

\section{Introduction}

\sechead{Variable Selection can be done via Penalties}
Over the past decades, statistical models have become so large and sophisticated that understanding their behavior has become a major challenge.
One way to improve model interpretability is to allow the model to have a large number of parameters, giving many possibilities before seeing the data, but then to eventually only use the most influential and set the majority to zero during the learning process.
In the context of linear regression, thresholding a coefficient to zero is equivalent to excluding the associated variable from the model. 
More complex models, if properly parameterized, can likewise become simpler to interpret if many of their parameters are identically zero.
One way to achieve zeroed-out parameter estimates is to use \textit{sparsifying} penalties which, when used to augment a loss function, lead to exact zeros at optimality.
Deonte our loss function by $\mathcal{L}$, the sparsifying penalty function by $g$, the parameter vector by $\b\in\mathbb{R}^P$ and define $\tau>0$ yielding:
\begin{equation}
\begin{split}
    \underset{\b\in\mathbb{R}^P}{\min}\,\, \mathcal{L}(\b) + \tau g(\b) \,\, .
\end{split}\label{eq:pencost}
\end{equation}
By varying $\tau$, we obtain a tradeoff between fitting the data and avoiding the penalty.
In this article, we will primarily be concerned with loss functions constituted by the negative log-likelihood of a given statistical model. 
However, the optimization machinery we consider in Section \ref{sec:optim} is more widely applicable to other problems of the form \ref{eq:pencost}.

A wide variety of penalty functions $g$ that induce sparsity in $\b$ have been studied (see Section \ref{sec:penbg}). 
Perhaps the most widely used is the $\ell_1$ penalty $\Vert \b \Vert_1 = \sum_{p=1}^P |\beta_p|$, which is referred to as \textit{Lasso regression} \citep{Tibshirani1996,taylor1979deconvolution}.
This may be viewed as independently placing a Laplace prior on $\beta_p$ and subsequently performing \textit{Maximum a Posteriori} (MAP) inference \citep{Tibshirani1996,Park2008}.
Beyond simply imposing that there be many zero parameters, we may also wish to impose rules stating that some parameters should be jointly zero or nonzero.
We refer to such a situation as \textit{structured sparsity}.
Extensions of the $\ell_1$ penalty allow for this, such as the Group Lasso \citep{yuan2006model,bakin1999adaptive}, which partitions the variables into jointly zero or non-zero sets.
The Sparse Group Lasso \citep{simon2013sparse} additionally allows for individual zero parameters within nonzero groups. 
In the context of ordered coefficients (associated with discrete time, say), the Fused Lasso \citep{tibshirani2005sparsity} imposes sparsity on their consecutive differences.
More general sparsity structures are available via graphs \citep{jacob2009group,huang2011learning} or trees \citep{jenatton2011structured,kim2012tree}.

In each case, a sparsifying penalty must be nondifferentiable.
But this complicates the iterative solution of Problem \ref{eq:pencost} in the general case. 
\textit{Proximal gradient} methods \citep{parikh2014proximal} allow for the solution of these problems without any loss in convergence rate and are easily implemented so long as a convex-analytic object known as the \textit{proximal operator} associated with $g$ may be efficiently treated numerically (see Section \ref{sec:proxbg}).
However, efficient computation of the proximal operator's action is essential to developing speedy proximal methods, and this is more difficult for certain sparsity structures than others.
For example, for Group Lasso, the case where an individual regression coefficient belongs to multiple groups, called the overlapping case, is more complicated than the non-overlapping case, as the overlap couples the proximal problems \citep{yuan2011efficient} and increases the cost of the computation of the proximal operator's action.
In practice, this necessitates schemes which duplicate variables \citep{obozinski2011group}.
Overlapping group Lasso is motivated by, for example, hierarchical second order sparse regression, where the main effects, interaction terms, and quadratic terms of a given pair of variables are jointly zero or nonzero.Thus, each main effect and quadratic term belongs to many groups.
Not all sparsity structures are guaranteed to have a proximal operator whose action can be easily computed, and such structures are inefficiently treated in practice by proximal gradient methods.

A separate issue with Lasso-like penalties are that they lead not only to sparsity in certain parameters, but also to shrinkage in the rest which can lead to significant bias \citep{Zhang2010}. 
One proposed solution is the Adaptive Lasso:
rather than a simple $\ell_1$ penalty, \citet{zou2006adaptive} proposes a weighted penalty $\Vert \b \Vert_1^{\l} = \sum_{p=1}^P \lambda_p |\beta_p|$, where $\lambda_p$ is a nonnegative weight associated with variable $p$. 
\citet{zou2006adaptive} proposes to fix the penalty coefficients to values determined in a preprocessing step, where a simpler estimator is used to form a rough estimate $\hat\beta_p$ of $\beta_p$, and then $\lambda_p = \frac{1}{|\hat\beta_p|^\gamma}$. 
Here $\gamma$ is a nonnegative hyperparameter.
This is simple enough for a linear model with more observations than parameters, but for a complicated likelihood, initial estimates may be difficult to come by.

\sechead{Our MAP-Bayesian approach to nonlinear adaptive lasso}
In this article, we will instead treat $\l$ in a MAP-Bayesian manner, endowing it with some hyperprior 
and jointly optimizing it together with $\b$:
\begin{equation} \label{eq:problem}
    \underset{\boldsymbol\beta\in\mathbb{R}^P,\boldsymbol\lambda\in\mathbb{R}_+^{P}}{\min}\,\, 
    \mathcal{L}(\boldsymbol\beta) + \sum_{p=1}^P \big[\tau\lambda_p |\beta_p| - \log\lambda_p\big] -\log p_\lambda (\boldsymbol{\lambda}) \,\,,
\end{equation}
where $p_\lambda$ is the density of a continuous joint prior for $\boldsymbol\lambda$, and the center term is the negative log of the Laplace distribution's density (with $\lambda_p$ serving as an inverse scale parameter for $\beta_p$ \textit{a priori}).
The main nondifferentiable component of Problem \ref{eq:problem} is given by the $\lambda_p|\beta_p|$ terms.
At first glance, it seems that we could just use the proximal operator associated with the absolute value function, which is known as the Soft Thresholding Operator \citep{Donoho1995}, to optimize this.
But if we are jointly optimizating $\b$ and $\l$, we actually need a new proximal operator associated not with the $\mathbb{R}\to\mathbb{R}_+$ function $f_1(\beta) = |\beta|$, but rather with the $\mathbb{R}\times \mathbb{R}_+\to\mathbb{R}_+$ function $f_2(\beta,\lambda) = \lambda|\beta|$, which is nonconvex. 
In Section \ref{sec:vsto}, we study this proximal operator and a related one, developing a simple closed form expression for their action.
This allows for efficient local optimization of Problem \ref{eq:problem} using a proximal gradient method so long as $p_{\lambda}$ is smooth (as is $\mathcal{L}$).

\sechead{You can have arbitrary structure with this one prox operator}
As we discuss in Section \ref{sec:stats}, using a sufficiently diffuse independent prior on $\l$ allows for variable selection on arbitrary likelihoods of sufficient regularity while obtaining the oracle property of \cite{fan2001variable}.
But we can also consider richer families of priors.
For instance, dependent priors on $\bl$ allow for operationalization of prior information regarding expected patterns of sparsity.
Take the Group Lasso.
We can express the prior information that might motivate this in our framework by placing a hierarchical prior on $\bl$, thereby encouraging penalty coefficients associated with regression coefficients from a given group to take on similar values.
But unlike the group Lasso, active groups will have their penalty coefficients shrunk to a small value, decreasing the bias present in the constant-penalty case.
In contrast to the existing state of the art, where increasingly complicated sparsity structures may lead to increasingly complicated proximal operators, by placing the structure in a smooth prior $P_\lambda$, we can impose arbitrary structure using the same proximal operator, in effect shifting the structure from the nonsmooth to the smooth part of the cost function.

To summarize, the major contributions of this article are as follows:
\begin{enumerate}
    \item We propose a one-algorithm-fits-all framework for optimization to perform debiased variable selection with general sparsity structures.
    \item We describe two simple proximal operators associated with the absolute value function.
    \item We show how to use these to conduct joint optimization for the adaptive Lasso with arbitrary smooth likelihoods.
    \item We demonstrate the validity of this procedure on complex sparsity structures, non-Gaussian likelihoods and on a battery of real and sythetic datasets and two social science applications (vaccination behavior and global migration).
    \item We show how to integrate this technique into a typical machine learning workflow using an automatic differentiation/linear algebra framework and demonstrate that it can be competitive in terms of execution time against model-specific methods while being broadly applicable.
\end{enumerate}

The remainder of this article is structured as follows.
Section \ref{sec:litreview} reviews relevant background before situating this study within the literature, developing a Bayesian perspective on the reweighted $\ell_1$ procedure of \cite{candes2008enhancing} along the way.
Section \ref{sec:optim} examines the optimization problem induced by MAP inference on the Adaptive Lasso and develops pertinent proximal operators.
Section \ref{sec:stats} determines the basic asymptotic properties of estimators based on this method while Section \ref{sec:priors} gives some examples of priors inducing popular sparsity structures.
Next, Section \ref{sec:applications} compares the procedure described here to previously proposed methods on a battery of real and synthetic datasets before showing how this approach can yield substantive conclusions on two social science case studies.
Finally, conclusions and future directions are presented in Section \ref{sec:conclusion}.

\section{Preliminaries and Background}\label{sec:litreview}

We begin this section with an overview of selected sparsifying penalties (Section \ref{sec:penbg}) before reviewing the coordinate descent and proximal gradient approaches to solving the associated optimization problems (Sections \ref{sec:cdbg} and \ref{sec:proxbg}). 
We then show how a popular approach to reweighting on-the-fly may be viewed as a coordinate descent procedure within the framework we study (Section \ref{sec:reweigh}).

\subsection{Penalty-Based Variable Selection}\label{sec:penbg}


\sechead{chapter contents} The importance of variable selection in nonlinear models has prompted decades of research and yielded a multitude of methods.  
We now review selected penalty-based approaches to variable selection.

\sechead{Adapt offline (before/after/between optim)} 
Some authors advocate for penalty wieghts adapted ``offline", that is, before or after the optimization process.
As previously mentioned, the Adaptive Lasso \citep{zou2006adaptive} specifies $\lambda_p=\frac{1}{|\hat{\beta}_p|^\gamma}$, where $\hat{\beta_p}$ is some initial estimate of the regression coefficient, though an initial estimate $\hat{\beta}_p$ is a nontrivial ask if $\mathcal{L}$ is complicated.
\citet{buhlmann2008} and \citet{candes2008enhancing} propose to iterate this procedure, updating the penalty coefficient $\lambda_p$ with new $\hat{\beta_p}$.
Alternatively, we may use Lasso for variable selection only, then proceed to an unpenalized procedure \citep{Efron2004,meinshausen2007relaxed,Zou2008}.

\sechead{Adapt via MCMC} 
We have seen that many penalties can be viewed from the MAP-Bayesian perspective, but there have also been more fully Bayesian approaches to variable selection. 
Indeed, several authors have proposed simulation-based procedures for adaptation of $\bl$.
\citet{kang2009self}, \citet{leng2014bayesian} and \citet{mallick2014new} place a Gamma prior or similar for $\lambda_p$ and conduct inference via Gibbs sampling.
The Horseshoe Prior \citep{Carvalho2010} specifies a conditional Normal prior for $\beta_p$ and a Half Cauchy prior on the the prior variances.
The Horseshoe prior enjoys generality as well as fast specific implementations such as those of \citet{Terenin2019} or \citet{Makalic2015}\footnote{The reader is referred to \cite{Bhadra2019} for an extensive comparison of the Lasso and Horseshoe models.}.
In a similar vein, \cite{bhattacharya2012bayesian} instead take the approach of specifying a Dirichlet prior on the regression coefficients.
Another approach is the Spike-Slab prior \citep{Mitchell1998} which uses discrete latent variables to categorize whether $\beta_p$ was \textit{a priori} sampled from the spike or the slab.
This complicates gradient-based inference. 
Somewhat in between Bayesian and penalized likelihood approaches lies Sparse Bayesian Learning \citep{tipping2001sparse}.
Here, sparsity comes not from nonsmooothness of the likelihood with respect to $\beta$, but rather by shrinking prior variance terms (and hence posterior variance terms) to zero via empirical Bayes.
As originally proposed, this requires a linear model and Gaussian error structure, though it can be expanded to Gaussian mixtures \citep{sandhu2021nonlinear}.
\citet{helgoy2019sparse} applied this framework to the Bayesian Lasso.

Nonconvex penalties do not vary the regularization strength but are directly constructed to impose minimal bias on nonzero coefficients, such as the Smoothly Clipped Absolute Deviation function \citep[SCAD]{fan2001variable,Hunter2005} or Minimax Concave Penalty \citep[MCP]{Zhang2010}.
These penalties do not correspond to proper priors, as they place constant, positive density arbitrarily far from the origin.
As with the $\ell_1$ penalty, imposing structured sparsity in these frameworks involves changing the nonsmooth penalty (and hence proximal operator).

Other penalties, exemplified by the ridge penalty \citep{hoerl1970ridge}, do not serve primarily to select variables, but rather to stabilize model fit or improve predictions.
These are not the primary focus of this article; however, they can induce structure when used in combination with nonsmooth penalties.
For example, \citet{slawski2010feature} propose an interesting approach to imposing structured sparsity, which, like our proposed method, can deal with generic sparsity structure while using only a single proximal operator (the standard Soft Thresholding Operator of Section \ref{sec:proxbg}, in their case).
In particular, they impose a standard $\ell_1$ penalty together with a structured $\ell_2$ penalty of the form $\b^\top\Sigma\b$, where the matrix $\Sigma$ encodes the group structure. 
The smooth penalty imposes a pattern in the coefficient magnitudes, with the $\ell_1$ penalty subsequently thresholding small parameters to zero.
However, this approach suffers from bias induced by both the $\ell_1$ and smooth penalties, while our approach allows for debiased structured sparsity.
Furthermore, the use of a Gaussian smooth structure is more restrictive than our general $\bl$ prior.
Indeed, the theoretical analysis in Section \ref{sec:stats} suggests that a Gaussian prior may not be sufficiently diffuse in our setting to allow for the oracle property.

\sechead{Structured Sparsity in Lasso:}
Oftentimes, sets of covariates are conceptually related, and it is desirable to impose sparsity at the group level.
Group Lasso \citep{yuan2006model} penalizes the $\ell_2$ norm of each group, imposing all zero or nonzero values for each coefficient of an entire group.
Sparse Group Lasso \citep{zhou2012modeling} allows individual coefficients within a nonzero group to themselves be zero.
 Group Lasso requires some special care in the case that groups overlap \citep{jenatton2011structured,bach2012structured}.
Fused Lasso \citep{tibshirani2005sparsity} allows for sparsity in ordered data. Extensions have been made to the spatiotemporal case for Lasso as well as for spike-slab models \citep{andersen2017bayesian}.
\cite{shervashidze2015learning} propose to use Variational Bayes to transfer sparsity structure between tasks in the multi-task setting;
 see \citet{gui2016feature} for a review of structured sparsity.
Finally, quite general structures can be imposed via graphs \citep{jacob2009group,huang2011learning} or trees \citep{jenatton2011proximal,kim2012tree}.
Just as the $\ell_1$ norm is a convex relaxation of the $\ell_0$ penalty, these structured penalties may be viewed as relaxations of combinatorial penalties \citep{obozinski2012convex}.
\citet{zhao2016bayesian} propose a simulation-based approach to structured sparsity.

\sechead{Compressed Sensing} 
Another perspective on sparsity comes from the \textit{compressed sensing} view \citep{gilbert2002near,candes2006robust}.
Group Lasso has been developed in the compressed sensing community \citep{ohlsson2010segmentation}.
\cite{ziniel2012generalized} develop a broad class of sparse signal decoding algorithms using the optimization procedure of \cite{rangan2011generalized} which allows for quite general structure on sparsity, but with a linear mixing assumption on the sparse signals. 
\cite{baraniuk2010model} study conditions under which we can expect recovery of structured sparse signals.
Several authors propose structured sparsity by specifying a latent indicator variable $z_i$ associated with each plausibly sparse parameter $\beta_i$ and endowing it with some joint prior, in the case of \citet{dremeau2012boltzmann} one similar to the Spike-Slab prior in the statistics community.
Powerful algorithms have been proposed for specific classes of problems \citep{rangan2011generalized,schniter2010turbo}, but optimization in the fully nonlinear case is complicated by the discrete nature of the indicator variables.
In contrast, we encourage structural sparsity through the continuous penalty coefficients of the $\ell_1$ norm.

\subsection{Coordinate Descent for Simple Likelihoods}\label{sec:cdbg}
The fastest commonly used algorithms for sparse inference in Generalized Linear Models (GLMs) are based on coordinate descent, notably the R package \texttt{glmnet} \citep{glmnet}. 
Intuitively, sparsifying penalty terms are inherently axis-aligned, and so if the likelihood term itself is sufficiently well-behaved, a similarly axis-aligned algorithm like coordinate descent can be quite fruitful.
\texttt{glmnet} uses coordinate descent to fit GLMs with simple sparsity and is quite efficient when the dataset can be fit in memory.
The key to fast coordinate descent algorithms is to skip coordinates which are going to be zeroed out \citep{tibshirani2012strong}, an approach expanded to the sparse group case by  \citet{ida2019fast}.

However, coordinate descent is most efficient when the coordinate subproblems are available in closed form (as in the Gaussian case) or at least can be efficiently solved using an iterative method like Newton's (as is the case for GLMs).
In this article, we will be concerned with general, possibly nonconvex likelihoods for which we cannot expect a straightforward solution to the coordinate subproblems.
Instead, we will investigate a joint optimization approach based on a gradient iteration.

\subsection{Proximal Gradient Algorithms for Smooth Likelihoods}\label{sec:proxbg}


We are interested in minimizing $c$, a loss function composed of a complicated but smooth term $l$ augmented with a simple but nonsmooth regularizer $g$: 
\begin{equation}\label{eq:auglik}
    c(\x) := l(\x) + g(\x)\, ,
\end{equation}
where $g$ is given by a multiple of the $\ell_1$ norm in the case of Lasso regression.
Proximal gradient descent and its relatives are generally applicable algorithms in such a situation. 
Given a function $g$ with domain $\mathcal{X}$ and some norm parameterized by a positive definite matrix $\mathbf{C}$, their \textit{proximal operator} is defined as the following $\mathcal{X}\to\mathcal{X}$ mapping:
\begin{equation}
    \mathrm{prox}_{g}^\mathbf{C}(\mathbf{x}) = \underset{\mathbf{u} \in \mathcal{X}}{\mathrm{argmin}} \, g(\mathbf{u}) + \frac{1}{2}||\mathbf{x}-\mathbf{u}||_{\mathbf{C}}^2 \,\, .
\end{equation}
We refer to solving this argmin as as the \textit{proximal problem}, and to the quantity being minimized as the \textit{proximal cost}.
Intuitively, the proximal operator of a function $g$ evaluated at a vector $\mathbf{x}$ returns another vector $\mathbf{u}$ which is close to $\mathbf{x}$ (wrt $\mathbf{C}$) but does a better job minimizing $g$. 

Proximal gradient algorithms optimize the objective function in Equation \ref{eq:auglik} via iteration of a two step process. Given a current solution $\mathbf{x}^k$ the next iterate is given by:
\begin{align}\label{eq:prox_descent}
    \hat{\mathbf{x}}^{k+1} = \mathbf{x}^k - \mathbf{C}^{-1}\nabla \mathcal{L}(\mathbf{x}^k) \\
    \mathbf{x}^{k+1} = \mathrm{prox}^{\mathbf{C}}_{ g}(\hat{\mathbf{x}}^{k+1})\, .
\end{align}
Note that we have not included a subgradient of $g$ in the gradient descent step.
The matrix $\mathbf{C}$ is defined by the preconditioning strategy and the step size. 

Proximal operators are most useful if they can be efficiently evaluated.
In this article, we will assume diagonal $\mathbf{C}=\mathrm{diag}(s_1,\ldots,s_P)$ (allowing for variable step sizes for each parameter) and an additive regularizer $g(\x) = \sum_{p=1}^P g_p(x_p)$, which is conducive to breaking the proximal problem into subproblems defined along each axis.
We will henceforth consider the case where $g_p$ are identical for all $p$ and, with a slight abuse of notation, simply denote this as $g$.

Different penalties are associated with different proximal operators.
For example, when $g:\mathbb{R}\to\mathbb{R}$ is given by $g(x)=\lambda|x|$, the proximal operator is given by the Soft Thresholding Operator (STO):
\begin{equation}\label{eq:soft.thresh}
    \mathrm{prox}^{s}_{\lambda |x|}(x) = (|x|- s \lambda)^+\mathrm{sgn}(x)
    \, ,
\end{equation}
where $(a)^+$ gives $\max(0,a)$ and $\mathrm{sgn}(a)$ gives the sign of $a$.
See \cite{bach2012optimization} for more background on optimization with sparsity-inducing priors.


\subsection{The MAP-Bayesian Perspective of Reweighted $\ell_1$ Penalties}\label{sec:reweigh}

The simplicity of computation with Lasso penalties has lead authors to develop algorithms for estimation with other penalties via linearization, where a sequence of reweighted Lasso problems are solved.
For example, \cite{candes2008enhancing} propose to iteratively set $\lambda_i^{k+1} = \frac{1}{|\beta_i^k|+\epsilon}$. 
Intriguingly, we find that we can fit this procedure within the framework of Problem \ref{eq:problem}.
In particular, where optimization is via block coordinate descent alternating between $\boldsymbol\beta$ and $\boldsymbol\lambda$ blocks and the prior for each $\lambda_p$ is exponential with rate $\epsilon$.
This is because, with other parameters fixed, the stationarity condition with respect to $\lambda$ is:
\begin{equation}
    \frac{\partial}{\partial\lambda_p}[-\log p(\beta|\lambda_p) - \log p(\lambda_p|\epsilon)] = \frac{\partial}{\partial\lambda_p} \big[\lambda_p|\beta_p| - \log\lambda_p + \epsilon\lambda_p\big] = (|\beta_p|+\epsilon) - \frac{1}{\lambda_p} = 0 \,.
\end{equation}
Whereas the $\epsilon$ term was initially heuristically introduced so as to avoid division by zero, we have provided here a MAP-Bayesian justification.
Additionally, the adaptive Lasso can be viewed as conducting one and a half steps of such a block coordinate descent procedure (two $\boldsymbol\beta$ optimization steps sandwiching a single $\boldsymbol\lambda$ optimization using the hyperparameters $\gamma=1$ and $\epsilon=0$).


\section{Joint Nonsmooth Optimization}\label{sec:optim}

%
This section studies the joint optimization of Problem \ref{eq:problem} via proximal gradient methods. 
We give only brief outlines of proofs in this section; see Appendix \ref{sec:app_proofs} for details.

\subsection{The Variable-Penalty $\ell_1$ Proximal Problem}\label{sec:vsto}

\begin{figure*}
	\centering
	\raisebox{-0.5\height}{\includegraphics[scale=0.58]{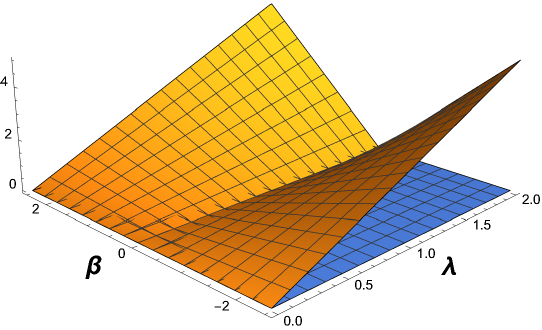}} \,\,
	\raisebox{-0.5\height}{\includegraphics[scale=0.48]{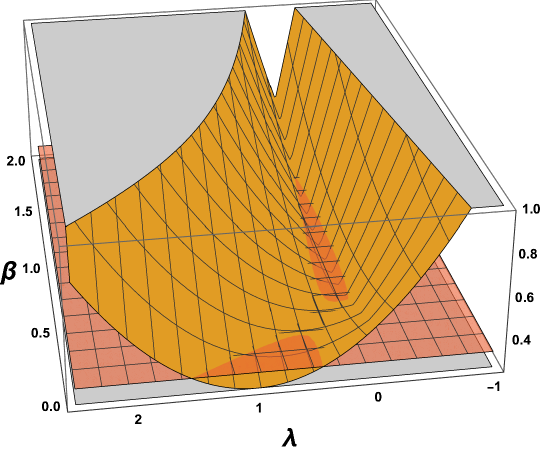}}
	\raisebox{-0.5\height}{\includegraphics[scale=0.48]{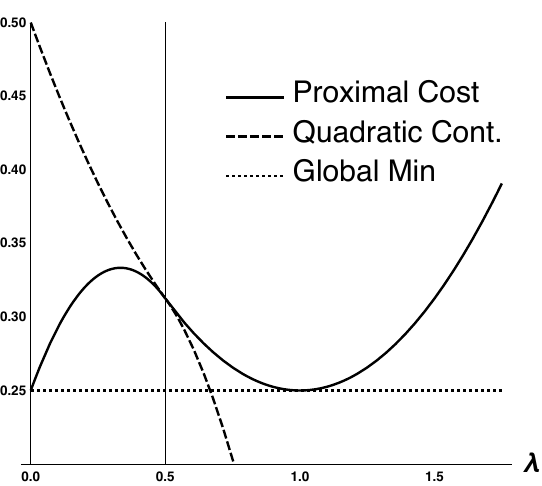}}
	\caption{
	\textit{Left:} The function $g(\beta,\lambda)=\lambda|\beta|$.
	\textit{Subsequently:} The proximal cost of as a function of $\beta$ and $\lambda$ (center)
	and marginal for $\lambda$ (right) with $\lambda_0 = \beta_0=1; s_\lambda=s_\beta = 2$.
	}
	\label{fig:nonconv}
\end{figure*}

Consider the proximal operator of the variable-coefficient $\ell_1$ norm function $g(\b, \boldsymbol\lambda)=\sum_{p=1}^P \lambda_p |\beta_p|$. 
Since this function decomposes into additive functions of each $(\lambda_p$, $\beta_p)$ individually, its proximal operator acts on each block independently of the others. 
Therefore, for the remainder of this section, we consider a single block $(\beta, \lambda)$, dropping the index $p$, and consider the proximal operator of the 2-dimensional function $g(\beta,\lambda) = \lambda |\beta|$.
As mentioned in the introduction, as the $\ell_1$ regularization coefficient $\lambda$ is now being optimized over, the STO is no longer the pertinent proximal operator. 
Indeed, the $\ell_1$ regularization function, considered formally as a function of both $\lambda$ and $\beta$, is nonconvex (see Figure \ref{fig:nonconv}, left), somewhat complicating proximal operator computation.
In fact, many authors, such as \citet{parikh2014proximal}, define the \texttt{prox} operator as one that acts on convex functions.
However, there has been work on extending proximality to larger classes of functions (see e.g. \citet{hare2009computing}) and, of particular interest to the statistical community, development of proximal operators for folded concave penalties such as SCAD, MCP or the bridge penalties $|\beta_p|^q$ for $q\in(0,1)$ \citep{marjanovic2013exact} which are not convex; see also \cite{mosci2010solving,polson2015proximal} for more on proximal methods in statistics and machine learning.
Perhaps because, aside from these discursions, of this focus on convexity (and despite the popularity of the $\ell_1$ norm and adaptive penalty methods), the proximal operator of $\lambda |\beta|$ as an $\mathbb{R}\times\mathbb{R}^+\to\mathbb{R}^+$ function has not to our knowledge been previously examined in the literature.
It turns out that the action of this proximal operator is available in closed form and is single-valued for almost all inputs and always for sufficiently small step sizes $s_\beta$ and $s_\lambda$ such that $s_\beta s_\lambda<1$.
We will assume in this section the overall regularization strength $\tau=1$, since a different $\tau$ simply scales the step sizes.
This yields the following proximal problem:

\begin{align*}\tag{P1}\label{eq:prox_prob}
     \mathrm{prox}^{s_\beta, s_\lambda}_{\lambda|\beta|} (\beta_0, \lambda_0) =
     \underset{x\in\mathbb{R},\lambda\geq0}{\mathrm{argmin}} \,\,  \lambda |\beta|+\frac{(\beta-\beta_0)^2}{2s_\beta} + \frac{(\lambda-\lambda_0)^2}{2s_\lambda} \,\, . & \hspace{5em}
\end{align*}

It will be illuminating to develop the marginal cost for $\lambda$ optimized over $x$.
\begin{lemma}\label{lem:marg}
    The marginal cost of \ref{eq:prox_prob} with respect to $\lambda$ (i.e. with $\beta$ profiled out) is the following piecewise quadratic expression:
\begin{equation}
     \underset{\lambda>0}{\mathrm{argmin}}  \begin{cases}
        \frac{1}{2}(\frac{1}{s_\lambda} - s_\beta)\lambda^2+(|\beta_0|-\frac{\lambda_0}{s_\lambda})\lambda + \frac{\lambda_0^2}{2s_\lambda} & \lambda < \frac{|\beta_0|}{s_\beta} \\
        \frac{(\lambda-\lambda_0)^2}{2s_\lambda}+\frac{\beta_0^2}{2s_\beta} & \lambda \geq \frac{|\beta_0|}{s_\beta}\,\, , \\
     \end{cases}
\end{equation}
where the changepoint $\lambda=\frac{|\beta_0|}{s_\beta}$ is the point where $\lambda$ is just large enough to push $\beta$ to zero.
\end{lemma}
\begin{proof}
    Convert to nested optimization and exploit the known solution for fixed $\lambda$.
\end{proof}
The quadratic polynomial in the interval $[\frac{|\beta_0|}{s_\beta},\infty)$ is always convex. When $s_\lambda s_\beta<1$, the quadratic polynomial in the other interval is convex as is the overall expression. But when $s_\lambda s_\beta>1$, the coefficient of the quadratic term is negative, and that polynomial is concave, yielding a nonconvex piecewise function (see Figure \ref{fig:nonconv}, center and right). 

We are now prepared to develop the proximal operator.


\begin{theorem}
    The optimizing $\lambda$ for the proximal program \ref{eq:prox_prob} is given by, when $s_\beta s_\lambda<1$:
    \begin{equation}\label{eq:prox1}
        \lambda^* =\begin{cases} 
          \lambda_0 & \lambda_0 \geq \frac{|\beta_0|}{s_\beta} \\
          \frac{(\lambda_0-s_\lambda|\beta_0|)^+}{1-s_\lambda s_\beta} & o.w. \,\,\,\, ,
       \end{cases} 
    \end{equation}
    and, when $s_\beta s_\lambda\geq1$, by $\lambda^* = \mathbbm{1}_{\big[\frac{\lambda_0}{\sqrt{s_\lambda}} > \frac{|\beta_0|}{\sqrt{s_\beta}}\big]} \lambda_0$, where $\mathbbm{1}$ denotes the indicator function.
In either case, the optimizing $\beta^*$ is subsequently given by $(|\beta_0|-s_\beta\lambda^*)^+\mathrm{sgn}(\beta_0)$.
\end{theorem}
\begin{proof}
    We need only compare the piecewise optima of the quadratic functions of Lemma \ref{lem:marg}.
\end{proof}


\begin{figure*}[h]
	\centering
    \includegraphics[width=0.24\textwidth]{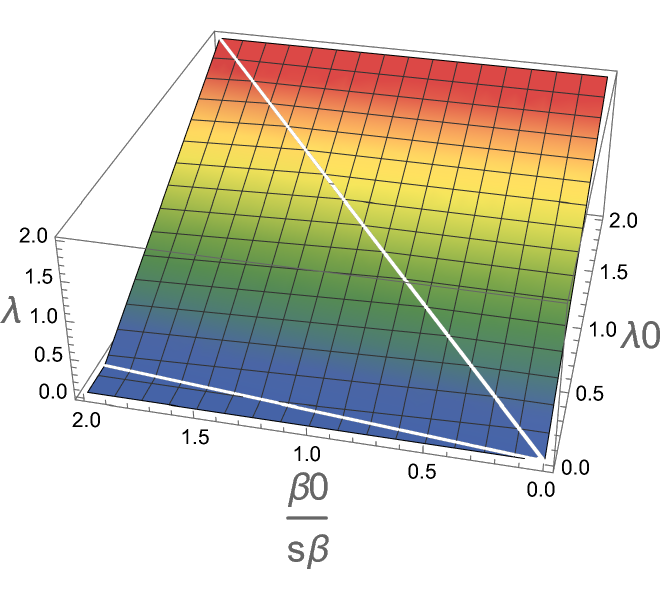}
    \includegraphics[width=0.24\textwidth]{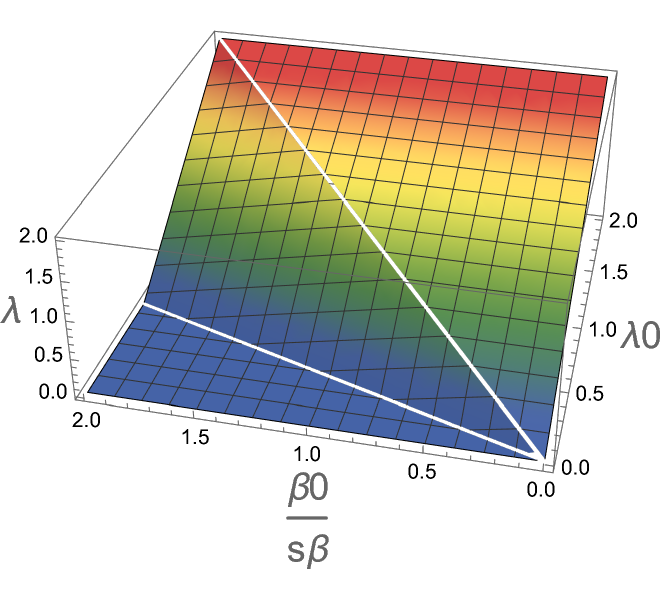}
    \includegraphics[width=0.24\textwidth]{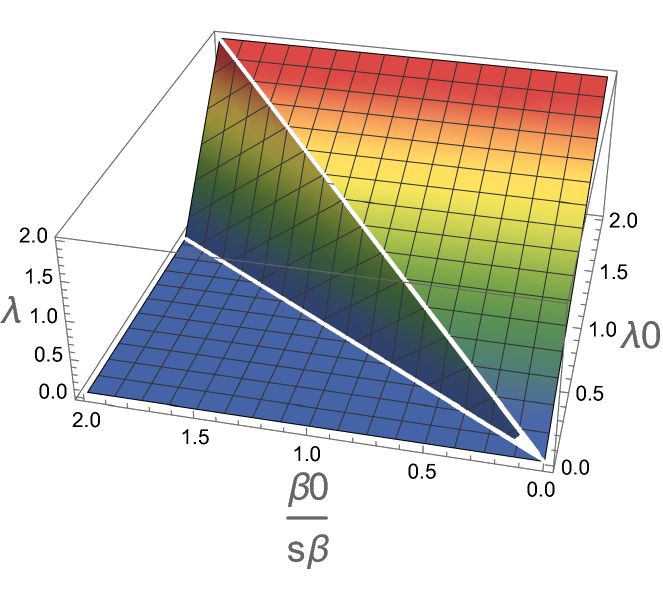}
    \includegraphics[width=0.24\textwidth]{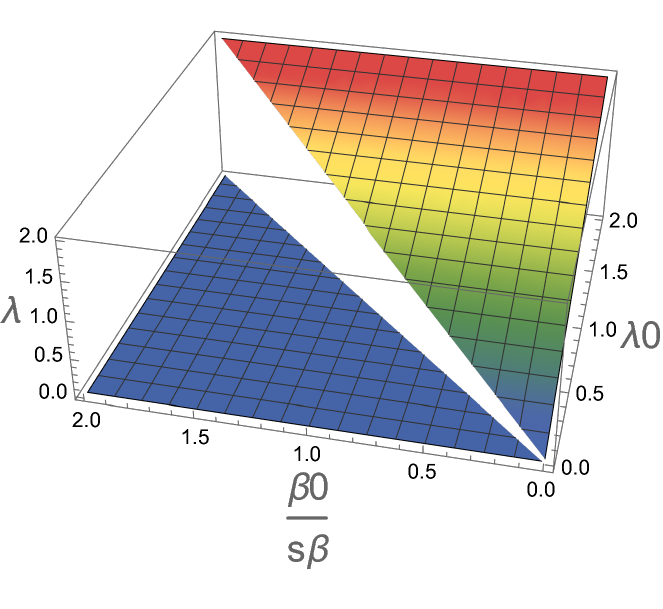}
    \caption{
        \textbf{The Action of the Proximal Operator:}
        Plots of the reduced proximal operator (Eq \ref{eq:reduced_prox}) for various fixed $b:=s_x s_\lambda<1$ and with $\lambda_0,\frac{|x_0|}{s_x} \in (0,2)$.
        Values $b=s_x s_\lambda\in\{0.1,0.35,0.65,0.99\}$ are shown left to right.
    }
    \label{fig:prox_act}

\end{figure*}

Due to the nonconvexity of the proximal cost, this proximal program may have two global optima. 
Thus the proximal operator is discontinuous and multi-valued at the discontinuity, as visualized in the top right of Figure \ref{fig:prox_act}.



\begin{remark}\label{rm:dualsparse}
    When $s_\beta s_\lambda<1$ and $\lambda_0 < s_\lambda |\beta_0|$ or when $s_\beta s_\lambda>1$ and $\frac{\lambda_0}{\sqrt{s_\lambda}}>\frac{|\beta_0|}{\sqrt{s_\beta}}$, the solution to the proximal problem gives $\lambda = 0$, which would lead to no shrinkage on the associated model parameter.
\end{remark}
Remark \ref{rm:dualsparse} is interesting, as it implies that it is possible to develop a procedure with ``dual sparsity": on the regression coefficient, when appropriate, or on the penalty coefficient.
However, in our application of this operator to the Laplace penalty, the $\lambda$ normalization term acts as a logarithmic barrier, precluding that point being a penalized maximizer.

This proximal operator has been conceptualized as a mapping of $(\lambda_0,\beta_0)\to(\lambda^*,\beta^*)$ parameterized by $s_\beta$ and $s_\lambda$, but to visualize it we will briefly study it as a function of these four quantities mapping to an optimizing $\lambda^*$.

\begin{remark}
Since when $s_x s_\lambda>1$ the $\lambda^*$ is either $0$ or $\lambda_0$, we will focus on the case where $s_x s_\lambda<1$. 
Then, let $a:=\frac{|x_0|}{s_x}$ and $b:=s_xs_\lambda$ yielding a function of just three variables:
\begin{equation}\label{eq:reduced_prox}
    \lambda(\lambda_0,a,b) = \begin{cases}
        \lambda_0 & \lambda_0 \geq a \\
        \frac{(\lambda_0-ab)^+}{1-b} & o.w. \,\,\,\,\, .
    \end{cases}
\end{equation}
\end{remark}
This function is visualized in Figure \ref{fig:prox_act}. 
For given step size product $b$ the function is stepwise linear, and converges to the identity mapping with respect to its $\lambda_0$ input as $b\to 0$. 
As $b\to 1$, the mapping becomes more and more steep for $\lambda_0\in(ab,a)$, gradually converging to the discontinuous mapping $\lambda(\lambda_0,a,b)\to \lambda_0 \mathbbm{1}_{\lambda_0 > a}$ .
The proximal operator is multi-valued at $\lambda_0=a$ when $b=1$.


\subsection{The Proximal Operator with Log Term}\label{sec:logprox}

In Equation \ref{eq:problem}, we actually have an additional nonsmooth term aside from the absolute value, namely $-\log\lambda$. 
Because we divide through by $\tau$, we will work with a slightly more general problem allowing an arbitrary coefficient in front of this log term.
In this section we'll thus consider the following proximal problem:
\begin{align*}\tag{P2}\label{eq:log_prox_prob}
     \mathrm{prox}^{s_\beta, s_\lambda}_{\lambda|\beta|+a\log\lambda} (\beta_0, \lambda_0) =
     \underset{\beta\in\mathbb{R},\lambda>0}{\mathrm{argmin}} \,\,  \lambda |\beta| + a \log \lambda + \frac{(\beta-\beta_0)^2}{2s_\beta} + \frac{(\lambda-\lambda_0)^2}{2s_\lambda} \,\, . & \hspace{5em}
\end{align*}

We develop the solution to this problem in the following theorem.

\begin{theorem}
    Consider Problem \ref{eq:log_prox_prob}.
    If $2\beta_0 \leq s_\beta \Big(\lambda_0^2 + \sqrt{\lambda_0^2+4 s_\lambda'}\Big)$, the solution is:
    \begin{equation}
        (\beta^*,\lambda^*) = 
        \Bigg(0,
        \frac{\lambda_0^2 + \sqrt{\lambda_0^2+4 s_\lambda a}}{2}
        \Bigg)\,.
    \end{equation}
    Otherwise, we have
    \begin{equation}
        \lambda^* = \frac{
        s_\lambda|\beta_0| - \lambda_0
        \pm\sqrt{
        (\lambda_0-s_\lambda\beta_0)^2 + 4(1-s_\beta s_\lambda) s_\lambda a
        }
        }{
        2(s_\beta s_\lambda-1)
        }\,
    \end{equation}
    and subsequently 
    $\beta^* = (|\beta_0|-s_\beta\lambda^*)^+\mathrm{sgn}(\beta_0)$.
\end{theorem}
\begin{proof}
    See Appendix \ref{sec:app_proofs}.
\end{proof}
We implement this proximal operator by computing both the positive and negative branches and then comparing the value of the proximal cost at each.
Conceptually, using this proximal operator means our surrogate precisely handles the log term, which is desirable.
But there is also an important practical gain: a negative $\lambda$ input to the proximal operator will come out strictly positive (instead of at zero, as in the prox without the $\log$ term), meaning that we can use an optimization method that does not respect the positivity of $\lambda$, such as vanilla gradient descent, and still rest assured that $\lambda$ will be strictly positive at each iterate after the proximal step.
For this reason, we use this proximal operator in our numerical experiments of Section \ref{sec:applications}.

\subsection{Incorporation into Proximal Gradient Methods}

The proximal operator investigated in Sections \ref{sec:vsto} and \ref{sec:logprox} can be deployed as part of a simple proximal gradient descent algorithm in our modeling context as follows, given some initial guesses $\boldsymbol\beta^1$ and $\boldsymbol\lambda^1$:

\begin{algorithm}[H]
    \For{$t\in\{1,\ldots,T\}$} {
         $\tilde{\boldsymbol\beta}^{t+1} \gets \boldsymbol\beta^t - s_{\boldsymbol\beta} \nabla_{\boldsymbol\beta} \mathcal{L}(\boldsymbol\beta^t, \boldsymbol\lambda^t)$ \\
         $\tilde{\boldsymbol\lambda}^{t+1} \gets \boldsymbol\lambda^t - s_{\boldsymbol\lambda} \nabla_{\boldsymbol\lambda} [\mathcal{L}(\boldsymbol\beta^t, \boldsymbol\lambda^t)-\log P_{\boldsymbol\lambda}(\boldsymbol\lambda^t)]$ \\
         $\boldsymbol\beta^{t+1},\boldsymbol\lambda^{t+1} \gets \mathrm{prox}_{\lambda|\beta|-\frac{1}{\tau}\log\lambda}^{s_{\boldsymbol\lambda},s_{\boldsymbol\beta}}(\tilde{\boldsymbol\beta}^{t+1},\tilde{\boldsymbol\lambda}^{t+1})$
    }
\end{algorithm}
Step sizes will need to be chosen carefully, as we discuss next.

\subsubsection{Convergence}


Since these proximal operators arise from nonconvex regularization terms, some standard convergence results for proximal gradient descent do not apply, as they typically assume the regularizer to be convex (see e.g. \cite{bauschke2011convex}).   
However, Theorem 5 of \citet{li2015global} applies to the proposed proximal iteration.
For convenience, we reproduce the theorem below using our notation from Equation \ref{eq:auglik} with added expository text and a reformatting for clarity:

\begin{theorem}[\citet{li2015global}]
    Suppose there exists a twice continuously differentiable convex function $q$ such that $-l\mathbf{I} \leq \nabla^2 h(\x) + \nabla^2 q(\x) \leq l\mathbf{I}$, that is, that the sum of the hessians of the smooth part of our objective $h$ and the postulated function $q$ has spectrum lying within $[-l,l]$ for any $\x$ in the domain.
    Then proximal gradient descent with a step size less than $\frac{1}{l}$:
    \begin{enumerate}
        \item Is a descent algorithm.
        \item Has cluster points only at stationary points.
    \end{enumerate}
\end{theorem}

This theorem does not imply that our algorithm will indeed have cluster points.
But if we do observe one, we know that it's also a stationary point of a sufficiently well-behaved objective.
We suspect that more precise statements could be made about the convex-analytic properties of these two new proximal operators and their implications for proximal gradient descent, which we leave as future work.

\subsubsection{Application to Large Finite Sum Problems}

A benefit of the proximal gradient framework is its modularity.
In our numerical experiments, we will consider some larger datasets for which stochastic methods can be advantageous.
In particular, we will use a stochastic proximal gradient method \citep{rosasco2014convergence} with variance reduction by way of SVRG \citep{johnson2013accelerating,xiao2014proximal}. 
Briefly, SVRG allows us to convert a stochastic finite sum problem into an increasingly deterministic one as iterations progress, with the cost of double the memory and triple the computation (per epoch) by comparison with standard stochastic gradient methods; see e.g. \cite{sebbouh2019towards}.

\section{Statistical Properties}\label{sec:stats}

In this section, we state our results about the asymptotic behavior of the procedure with an unstructured prior; see Appendix \ref{sec:asymp_appendix} for the corresponding proofs.
We will assume that the prior on $\boldsymbol\lambda$ is independent and that the misfit term $\mathcal{L}(\bb)$ is given by a negative loglikelihood $-l(\y;\bb)$.
We will assume all parameters are penalized for presentation purposes but this does not impose any serious constraints.
In this setting, we can rewrite our penalty as such:
\begin{align}
    & \underset{\boldsymbol\beta,\boldsymbol\lambda>\mathbf{0}}{\min}\,\, -l(\boldsymbol\beta) + \sum_{p=1}^P \big[\tau\lambda_p |\beta_p| - \log\lambda_p\big] + \sum_{p=1}^P -\log p_\lambda (\lambda_p) \\
    & \iff \underset{\boldsymbol\beta}{\min}\,\, -l(\boldsymbol\beta) + \sum_{p=1}^P \underset{\lambda_p>0}{\min} \big[\tau\lambda_p |\beta_p| - \log\lambda_p  -\log p_\lambda (\lambda_p) \big]  \label{eq:penlik} \, .
\end{align}
Therefore, we may profile over $\lambda$ to develop a penalty $g_{\tau}(|\beta|) := \underset{\lambda>\mathbf{0}}{\min} \, \tau\lambda |\beta| - \log\lambda  + \rho(\lambda)$, where $\rho(\lambda) := -\log p_\lambda (\lambda)$.
We will build off of the results of \cite{fan2001variable}.
However, one complication in this article is that we do not have a direct expression for the penalty function; rather, it is defined implicitly as a solution to an optimization problem given $|\beta|$. 
We thus need to establish some basic properties of the penalty before establishing any statistical results.
\begin{lemma}\label{lem:lam_props}
	The following hold, where $\lambda^*$ denotes the optimizing $\lambda$, and is formally a function of $\tau$ and $|\beta|$:
 \begin{multicols}{2}
	\begin{enumerate}
		\item $\lambda^* = \frac{1}{\tau|\beta| + \rho'(\lambda^*)}$.
		\item $\frac{\partial \lambda^*}{\partial |\beta|} = -\frac{\tau}{\frac{1}{\lambda^{*2}} + \rho''(\lambda)}$.
		\item $g_{\tau_n}'(|\beta|) =  \tau \lambda^*$.
		\item $g_{\tau_n}''(|\beta|) =  -\frac{\tau_n^2}{(\tau_n+\rho'(\lambda^*))^2+\rho''(\lambda^*)}$.
	\end{enumerate}
 \end{multicols}
\end{lemma}
\begin{proof}
    These follow from implicit differentiation on first order optimality conditions.
\end{proof}
This allows us to quantify the behavior of this penalty as follows:
\begin{theorem}
    Assume that the logarithmic derivative of the hyperprior density on $\lambda$ is bounded (i.e. $|\rho'(\lambda)|< M_1 \,\, \forall \lambda\geq0$) and that the density is decreasing on $(0,\infty)$. Then:
    \begin{enumerate}
        \item $g'_{\tau}(|\beta|)\approx\frac{1}{|\beta|}$ for large $\beta$.
        \item The minimum of $|\beta|+g'_{\tau}(|\beta|)$ is achieved at $\beta=0$ with value $\lambda_a\tau$.
    \end{enumerate} 
\end{theorem}

\begin{remark}
    $g'_{\tau}(|\beta|)\approx\frac{1}{|\beta|}$ approximates the gradient of the adaptive Lasso \citep{zou2006adaptive} procedure with optimal weights with hyperparameter $\gamma=1$ as well as the iteration described by \citet{candes2008enhancing}.
\end{remark}
\begin{remark}
    \citet{fan2001variable} describe three desirable properties of nonconcave penalties: first, that they be unbiased for large parameters, second, that they induce sparsity, and third, that they be continuous in the data. 
    Unlike the penalty they study, ours is not unbiased for large parameters, but is simply asymptotically unbiased.
    The fact that the minimum of $|\beta|+g'_{\tau}(|\beta|)$ is strictly positive ensures sparsity, while the fact that the minimum occurs at zero ensures continuity. This latter condition is not satisfied by, for example, the bridge penalty.
\end{remark}
\begin{remark}
Bounded logarithmic derivatives are satisfied by the densities of, for example, the Cauchy and Exponential distributions, but not the Gaussian distribution.
\end{remark}

We next consider the asymptotic distribution of the penalized likelihood estimator.
In particular, we find that there exists a local minimizer of the penalized loss which satisfies the oracle property of \citet{fan2001variable}.
Without loss of generality, we assume that it is the first $r$ entries of the true parameter vector $\tb$ which are nonzero, and the rest 0, such that $\tb={\tiny\begin{bmatrix} \tb_{1}\\ \tb_{2}\end{bmatrix}=\begin{bmatrix}\tb_{1}\\\mathbf{0}\end{bmatrix}}$.
\begin{theorem}
    Let $\tau_n$ be some linearly increasing sequence, and further assume that $|\rho''(|\lambda|)|<M_2$ (bounded second logarithmic derivative). Then, under the standard regularity conditions on the likelihood enumerated in the supplementary material, there is a local minimum of $\ref{eq:penlik}$ that satisfies the following:
    \begin{enumerate}
        \item For any $\epsilon>0$, $P(\hat{\bb}_2=\mathbf{0})>1-\epsilon$ as $N\to\infty$.
        \item $\sqrt{N}(\hat{\bb}_1-\tb_1) \overset{P}{\to} N\big(\mathbf{0},I^{-1}(\tb_1)\big)$.
    \end{enumerate}\label{thm:oracle}
\end{theorem}
The oracle property tells us that this estimator has the same asymptotic distribution as that estimator with truly zero $\beta_j$ clamped to zero.

\section{Building in Structure with $\lambda$ Priors}\label{sec:priors}

Existing nonconvex alternatives to our proposed biconvex penalty in the form of the MCP, SCAD and Bridge penalties all come with an extra hyperparameter controlling how close the penalty gets to approximating $||.||_0$ directly.
The proposed approach retains only the single global penalty strength parameter, which in practice means it is possible to simply try all pertinent parameter values via a sequence of warm starts.
However, it does require specification of a prior on $\lambda$.
We now overview some choices for $p_{\boldsymbol\lambda}(\lambda)$ used in our numerical experiments

\subsection{Independent Prior for Unbiased Model Selection}\label{sec:prior_indep}

Inspired by the Horseshoe Prior, we use the half Cauchy $p_{\lambda}(\lambda) = \mathcal{C}(0,  1)^+$ as a baseline adaptive prior, i.e. if there is no \textit{a priori} understood structure in the sparsity.
This allows for large variation in penalty strength and satisfies the decay rates of Theorem \ref{thm:oracle}.

\subsection{Sparse Group Lasso}\label{sec:prior_group}

Let $\mathcal{P} := \{1, \ldots, P\}$ denote the index set of predictor variables.
Then we define a grouping of variables as a partition of this set, $\mathcal{G}_g \subset \mathcal{P}$ with $\cup_{g=1}^G \mathcal{G}_g = \mathcal{P}$ and $\mathcal{G}_i \cap \mathcal{G}_j = \emptyset$ for $i\neq j$.
Denote $g(p)$ as the function which maps a given predictor to the group it belongs to.
One way to formulate the notion that coefficients in the same group should have a shared sparsity is to impose a hierarchical prior with shared parameters within groups.
In particular, we define group-level variables $\gamma_{g}$ with priors $p_\gamma = \mathcal{C}(0,1)^+$ and subsequently $\lambda_i|\gamma_{g(i)} \sim \mathcal{N}(\gamma_{g(i)}, \frac{1}{\sqrt{N}})^+$.
This prior is motivated by the idea that important variables in a given group will pull down that groups $\gamma_g$, leading the $\lambda_i$ for variables in that group to shrink as well and encouraging nonzero $\beta_i$.
We found empirically that a factor of $\frac{1}{\sqrt{N}}$ provided the right balance between the group structure and the likelihood.

\subsection{Overlapping Group Lasso}\label{sec:prior_hier2nd}

Building on the previous subsection, we now consider the case where the sets $\mathcal{G}_g$ overlap. 
In previous literature, authors have been able to bring non-overlapping methods to bear on such overlapping problems by duplicating variables, e.g. making a copy of each main effect for each interaction term it belongs to; one early article implementing this strategy is \cite{obozinski2011group}.
There is no need for such duplication when using our approach.
$g()$ should now be understood as a set valued function, returning the set of groups to which the covariate $p$ belongs.
A variable $\beta_p$ should be encouraged to have a nonzero value if any of its groups are active, i.e. if $\gamma_{i}$ is small for any $i\in g(p)$.
This could be obtained by setting $\lambda_p\sim C\big(\min_{j\in g(p)} \gamma_p, \frac{1}{\sqrt{N}}\big)^+$. 
However, in order to obtain a differentiable prior density, we used $\lambda_p\sim C\big(\sigma(-\frac{\boldsymbol\gamma_{g(p)}}{\sqrt{P}})^\top \boldsymbol\gamma_{g(p)}, \frac{1}{\sqrt{N}}\big)$ where $\sigma$ is the softmax function such that $\sigma(-\frac{\boldsymbol\gamma_{g(p)}}{\sqrt{P}})^\top \boldsymbol\gamma_{g(p)}$ is a smooth approximation to $\min_{j\in g(p)} \gamma_{j}$.
We found empirically that using a temperature of $\sqrt{P}$ in the softmax lead to better performance.

\section{Applications and Numerical Experiments}\label{sec:applications}

We now present two sets of numerical experiments and two case studies. 
For the first of these, we compare our methodology to existing programs on Gaussian and non-Gaussian regressions with four likelihoods on synthetic datasets in Section \ref{sec:exp_synthetic}. 
Section \ref{sec:exp_libsvm} contains a similar comparison but on real data from the UCI repository\footnote{\url{https://archive.ics.uci.edu/}}.
Next come our case studies: first dimension reduction via neural network for understanding vaccination behavior, presented in Section \ref{sec:mosaic}, and finally a model of international human migration (Section \ref{sec:hurdle}).
But first we go into further detail about each of these studies in Section \ref{sec:exp_design}

\subsection{Experimental Design and Modeling Details}\label{sec:exp_design}


\subsubsection{Implementation}

We implement our proximal operators in Python \footnote{Code is available in a Git repo at \url{https://github.com/NathanWycoff/prox_alasso}.} using JAX \citep{jax2018github} and use Tensorflow Probability \citep{dillon2017tensorflow} to define the likelihood functions.
We initialize $\b=\mathbf{0}$ and $\bl=\mathbf{1}$ and set $\mathbf{C}$ using the update rule from the Adam optimizer \citep{kingma2014adam} with a step size of 1e-2. 
We use a minibatch size of 256 in all experiments.
To assess convergence, we used early stopping with a patience of 500 iterations and with $1,000$ randomly sampled observations held out. 

\subsubsection{Non-Gaussian Regression Studies}

For the regression comparisons of Sections \ref{sec:exp_synthetic} and \ref{sec:exp_libsvm}, we compare the runtime and MSE of our JAX program against existing packages on various datasets.
For each likelihood, we use as a first baseline an unpenalized regression:
we use \texttt{statsmodels} \citep{seabold2010statsmodels} maximum likelihood inference for the normal, binomial and Negative Binomial likelihoods and a Huber robust regression for the Cauchy likelihood.
Additionally, we use \texttt{glmnet} \citep{friedman2010regularization} for penalized GLMs with Gaussian, Poisson or Bernoulli likelihood, choosing model complexity via the built-in cross validation function.
For a structured regularization competitor, we consider \citet{grimonprez2023mlgl}, who propose a method for overlapping group Lasso for Gaussian and Bernoulli problems which is implemented in the \texttt{R} package \texttt{MLGL} (for ``Multi-Level Group Lasso"), using BIC to select the optimal regularization strength.

For the synthetic case studies, the $\mathbf{X}$ matrix has entries independently sampled from a standard normal distribution.
We set $N=10,000$ and aim for approximately $1,000$ columns in $\X$; this is achieved exactly in the independent and group sparsity cases and approximately in the hierarchical case by setting the number of variables to 45.
Nonzero regression coefficients are sampled independently from a standard normal distribution, and the data sampled from the resulting model.
For the independent sparsity case, we set 10 parameters to nonzero values, while for the grouped cases, we set 10 groups to jointly nonzero values.

We measure the MSE in terms of estimating the regression coefficients $\b$ for the synthetic case study while for the real data we measure out of sample prediction MSE.
We also measure elapsed real time and repeat each simulation setting 30 times. 
$\texttt{glmnet}$ does not support Cauchy or Negative Binomial likelihoods, so we excluded it from the Cauchy study and used it in Poisson mode for the Negative Binomial study.
MLGL supports only Gaussian and Bernoulli likelihoods, so we excluded it from the Negative Binomial and Cauchy studies.
We use a regularization strength of $\tau=0.025 N$ for the independent case and $\tau=0.015 N$ for the structured case.

For the real data, the 8 test problems from the UCI repository are given in Table \ref{tab:libsvm_problems}.
To evaluate predictive error, we randomly split the dataset 50-50 (following \cite{ida2019fast}) and evaluate the prediction error on the held out half.

\subsubsection{Vaccination Behavior Case Study}

During the COVID-19 pandemic, the major challenge faced by public health authorities in developing a vaccinated population was not the technical challenge of creating vaccines, but the socio-political challenge of getting them into arms.
To study this, we conducted a proability survey of Americans, asking for both basic demographic information and vaccination status for each respondent \citep{singh2024understanding}.
For a subset of respondents who used Twitter\footnote{This study was conducted prior to the change in ownership of that platform.}, we asked them if they would share their handle and collect their data to support computer science and social science research \footnote{IRB STUDY00003571.}.
Our subsample who consented to provide their Twitter handle is of size 425, and for this subsample we determined which accounts they followed via the Twitter API.
We grouped respondents into three categories: ``Early Adopter" if they indicated before the vaccine was widely available that they intended to get it, ``Vaccinated Skeptic" if they initially indicated that they were unsure or did not intend to be vaccinated but eventually did receive a vaccine, and ``Persistent Antivaxer" if they initially indicated that they were unsure or did not intend to vaccinate and were indeed not vaccinated during the study period (March 1 2021 to February 9 2022) \footnote{The fourth possible configuration, of initialy indicating that one would receive the vaccine but failing to follow up on this, consisted only of two respondents who we consquently dropped from the analysis.}.

To gain some insight into what variables are correlated with vaccination behavior, we fit a Deep Active Subspace Classifier (DASC) \citep{tripathy2019deep,edeling2023deep} with a group penalty. 
A DASC  is a multi-layer neural network whose first hidden layer is restricted to have a low dimension; in this study we chose dimension 2 to enable visualization.
Let us denote the weight matrix associated with this layer by $\mathbf{W}_1 \in \mathbb{R}^{2\times P}$, where $P=65$ is the dimension of our input space, consisting of 51 indicator variables encoding whether a given respondent followed a given Twitter account and 14 demographic variables.
The rest of the network architecture is unconstrained, but by forcing a compressed representation we gain some regularization; we use one additional hidden layer of size 512.
By analogy with Sparse PCA methods \citep{zou2006sparse}, we define a Sparse DASC as one with a sparse input weight matrix. 
We used groups of size two, containing both weights mapping a given variable into the two dimensional active subspace, such that a variable is encouraged to be either included in both factors or neither.

\subsubsection{International Migration Case Study}

In this section, we study a dataset of global migration\footnote{We gratefully acknowledge the United Nations High Commissioner for Refugees (UNHCR; \url{https://www.unhcr.org/us/}) for these data.} containing asylum claims filed by individuals moving between 194 countries pairwise and yearly from 2000 to 2021, resulting in 823,724 observations.
Approximately 88\% of the responses variable consists of zeros and the mean number of asylum claims between all pairs at all time points is approximately 56.
To explain this movement data, we use country-level indicators developed by \citet{mayer2011notes},
including 14 variables recorded at the origin\footnote{Namely: population, Consumer Price Index, GDP, GDP at PPP, island Indicator, area of the country, landlocked indicator, Political Repression Index, Civil Liberties Index, Violence Index, estimate of deaths, conflict indicator, length of conflict and it logarithm.}, the same 14 variables at the destination and 8 dyadic variables\footnote{``Dyadic" meaning they are associated with a pair of countries, in particular: neighboring country indicator, common language indicator, common ethnicity indicator, colonial ties indicator (e.g. Britain and Canada), common colonial ties indicator (e.g. Canada and India), colonial ties post 1945 indicator, formerly same country indicator}  giving 36 variables encoding country properties.
Commonly, such data are interpreted in the framework of the \textit{gravity model}.
The gravity model specifies that the exchange between two regions decreases with the distance between them by analogy with the physical law of gravity \citep{tinbergen1962shaping}.
Often, these models are estimated by taking the log of both migration and distance and proceeding either by least squares or via a count model \citep{silva2006log}.
Other factors such as recorded deaths or population also influence the expected exchange and are used in the model.
Furthermore, we expect there to be individual effects associated with countries being more or less resistant to migratory influx versus outflux.
The model captures these via what are called \textit{bilateral resistance to exchange} terms $\omega_{i,j,t}$ which are typically assumed to factor into dummy variables associated with each origin, destination and time period \citep{anderson2003gravity}, i.e. $\omega_{i,j,t} = \omega_i + \omega_j + \omega_t$.
From a classical statistical perspective, this is mathematically equivalent to treating origin, destination and year as blocking effects.
We incorporate these indicator variables in our model alongside the country property variables, leading to 446 total main effects.

We wish to maintain the straightforward interpretability of the standard gravity model.
In order to account for the preponderance of zeros in our dataset, we upgrade this model to a hurdle model \citep{cragg1971some}, which specifies a second set of regression coefficients governing whether or not a given pair of countries will have any exchange at all during a given year:
$z_{i,j,t} \sim Bern(\sigma(\x_{i,j,t}^\top \boldsymbol\beta_\pi))$
where $\sigma$ is the logistic sigmoid function. 
If $z_{i,j,t}=1$, then $y_{i,j,t}=0$. Otherwise, $y_{i,j,t} \sim NB(s(\x_{i,j,t}^\top\boldsymbol\beta_\mu),\alpha)$, where $s$ is the appropriate link function and $\alpha$ controls overdispersion.
However, the assumption that the effect of various factors on migration is constant throughout countries and time is unrealistic.
We therefore entertain a model which considers interaction terms between our covariates and our migration resistance terms $\boldsymbol\omega$, as well as between the $\boldsymbol\omega$ terms themselves. 
This leads to model with $2\times(446 + {446\choose 2})=99,681$ variables and twice that many parameters (excluding $\alpha$), which we will tame with our proposed group sparsity strategy.
In particular, building on the overlapping group lasso of Section \ref{sec:prior_hier2nd}, we place overlapping group priors with a group corresponding to each interaction term both within the mean term $\b_\mu$ and the zero term $\b_\pi$:
\begin{align}
    &\gamma_{i,j} \sim C^+(0,1) \textrm{ for } i\neq j\in \{1, \ldots, P\} \hspace{1em} \textrm{(Group Penalty Strength Prior)}\\
    &\lambda_{\mu,i},\lambda_{\pi,i} \sim C \big(\sigma(-\frac{\boldsymbol\gamma_{g(i)}}{\sqrt{P}})^\top \boldsymbol\gamma_{g(i)}, \frac{1}{\sqrt{N}}\big)  \hspace{1em} \textrm{(Interaction Terms)}\\
    &\lambda_{\mu,i,j},\lambda_{\pi,i,j} \sim N(\gamma_{i,j}, \frac{1}{\sqrt{N}}) \textrm{ for } i\neq j \in \{1, \ldots, P\} \hspace{1em} \textrm{(Main Effects)}
\end{align}

Forming the entire $823,724\times 99,681$ design matrix in 64 bit precision would take about 650 gigabytes, rendering direct inference infeasible.
However, we can deploy a stochastic variant of our proximal gradient algorithm by building the interaction terms on the fly within a minibatch. Denoting $\boldsymbol\theta := [\b_0, \b_m, \alpha]$, we iterate:
\begin{center}
\begin{algorithm}[H]
    \For{$t\in\{1,\ldots,T\}$} {
        $\textrm{subsamp} \gets \textrm{Subsample}(N,m)$ \\
        $\X^s \gets \X[\textrm{subsamp},]$ \\
        $\y^s \gets \y[\textrm{subsamp}]$ \\
        $\X^q \gets \textrm{add\_int}(\X^s)$\\
        $\mathbf{g}^t \gets \hat{\nabla f}(\X^q,\y^s, \boldsymbol\theta^t)$ \\
        $\boldsymbol\theta^{t+1} \gets \textrm{prox\_sgd}(\mathbf{g}^t, \boldsymbol\theta^t)$
    }
\end{algorithm}
\end{center}
Here, the function $\textrm{add\_int}$ calculates the interaction terms for a submatrix (taking $\mathbb{R}^{m\times 446}\to \mathbb{R}^{m\times 99,681}$), and $\textrm{prox\_sgd}$ implements one step of the proximal stochastic gradient descent algorithm \citep{rosasco2014convergence} using the proximal operator of Section \ref{sec:logprox}.
To compare this model with varying penalty strengths predictively, we randomly held out half of the dataset and evaluated out-of-sample accuracy in terms of predictive log likelihood. 

\newcommand{\sscale}{0.5}
\begin{figure}
    \centering
    \includegraphics[width=0.9\textwidth,trim= 10em 0 10em 0,clip]{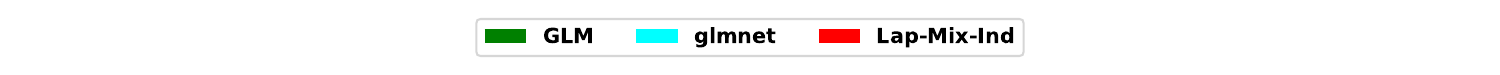}
    \includegraphics[width=0.9\textwidth]{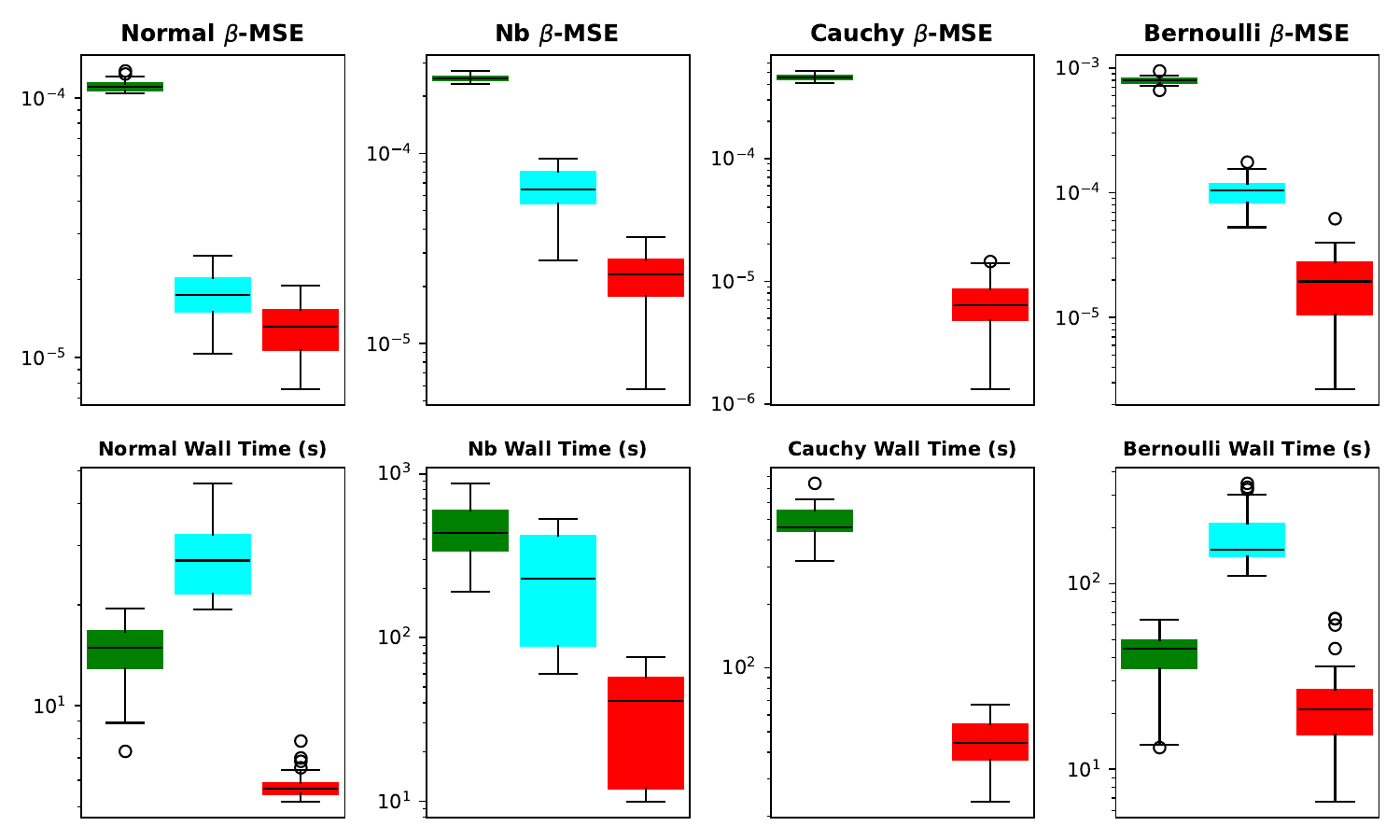}
    \caption{Comparison on synthetic data with independent sparsity.}
    \label{fig:synth_random}
\end{figure}

\begin{figure}
    \centering
    \includegraphics[width=0.9\textwidth,trim= 7em 0 7em 0,clip]{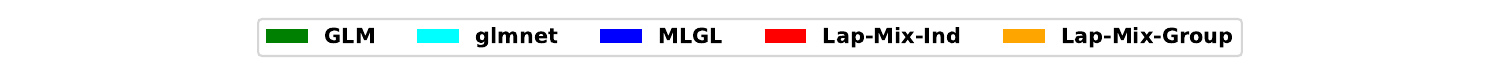}
    \includegraphics[width=0.9\textwidth]{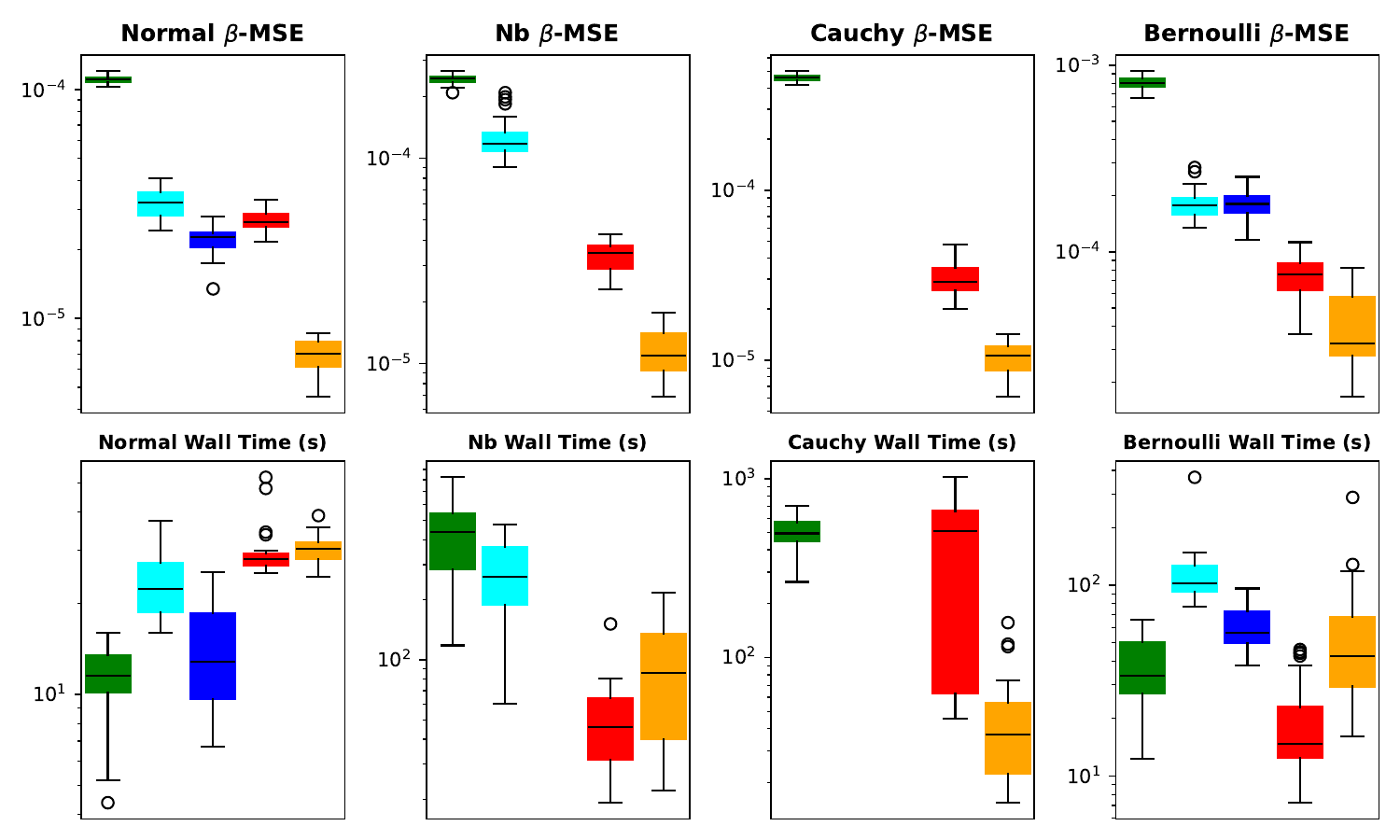}
    \caption{Comparison on synthetic data with group sparsity.}
    \label{fig:synth_group}
\end{figure}

\subsection{Synthetic Regression Simulation Study Results}\label{sec:exp_synthetic}

We now present the results of our regression study on synthetic data.


\begin{figure}
    \centering
    \includegraphics[width=0.9\textwidth,trim= 7em 0 7em 0,clip]{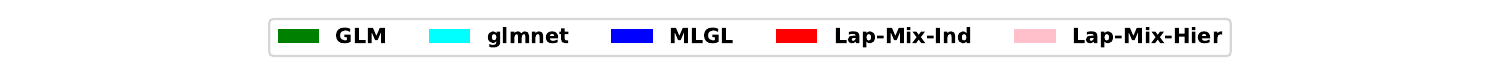}
    \includegraphics[width=0.9\textwidth]{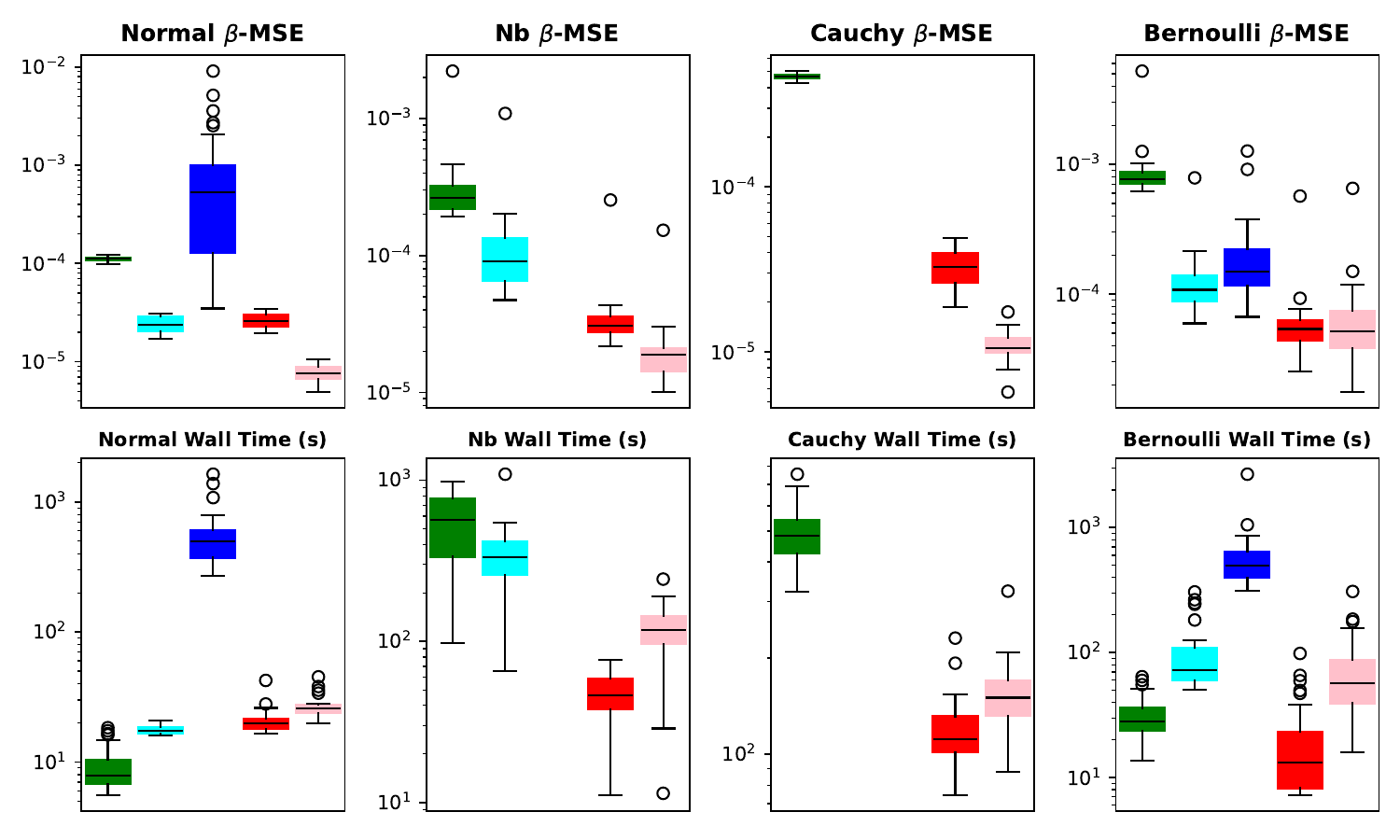}
    \vspace{-1em}
    \caption{Comparison on synthetic data with hierarchical sparsity.}
    \label{fig:synth_hier}
\end{figure}

\subsubsection{Simple Variable Selection} 
We begin with simple sparsity (Figure \ref{fig:synth_random}), that is, sparsity without structure and use the prior of Section \ref{sec:prior_indep}. 
We find that the proposed methodology is able to achieve lower MSE and faster median execution speed than \texttt{glmnet} or unpenalized regression across the normal, negative binomial, and Bernoulli settings.
Furthermore, it is able to seamlessly handle the Cauchy likelihood, which would completely throw off a non-robust regression method like \texttt{glmnet} and at higher accuracy than a standard robust Huber regression which does not exploit sparsity.

\subsubsection{Grouped Variable Selection} 
Next (Figure \ref{fig:synth_group}), we consider groups of size 5 and use the prior of Section \ref{sec:prior_group}.
We find that the grouped method outperforms all other methods in terms of MSE, including our own simple sparsity method.
Furthermore, the execution time is comparable to  other regularized methods on the Gaussian problem, and superior on non-Gaussian likelihoods.

\subsubsection{Hierarchical Variable Selection} 
In our final synthetic simulation study (Figure \ref{fig:synth_hier}), we fit second order models allowing for both interaction and quadratic terms. 
In this case, there is one group associated with each interaction term, and hence there is significant overlap in group membership for main and quadratic effects.
We find that our hierarchical approach prevails in terms of MSE, while offering about an order of magnitude faster execution time than \texttt{MLGL}.

\begin{table}[]
    \centering

    \begin{tabular}{|lrrcl|}
    \toprule
    Problem & N & P & $2P+ {P \choose 2}$ & Likelihood \\
    \midrule
    Obesity & 2111 & 23 & 299 & Normal \\
    Aids & 2139 & 23 & 299 & Bernoulli \\
    Rice & 3810 & 7 & 35 & Bernoulli \\
    Abalone & 4177 & 8 & 44 & Normal \\
    Dropout & 4424 & 36 & 702 & Bernoulli \\
    Spam & 4601 & 57 & 1710 & Bernoulli \\
    Parkinsons & 5875 & 19 & 209 & Normal \\
    Shop & 12330 & 26 & 377 & Bernoulli \\
    \bottomrule
    \end{tabular}
    
    \caption{Problem Summaries; $2P + {P\choose 2}$ gives the size of the second order model.}
    \label{tab:libsvm_problems}
\end{table}

\begin{figure}
    \centering
    \includegraphics[width=0.9\textwidth,trim= 7em 0 7em 0,clip]{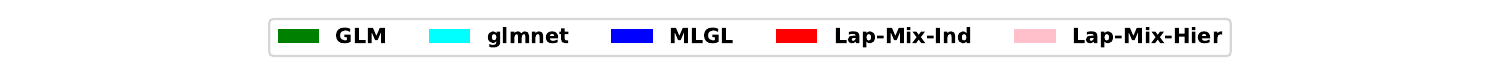}
    \includegraphics[width=0.9\textwidth]{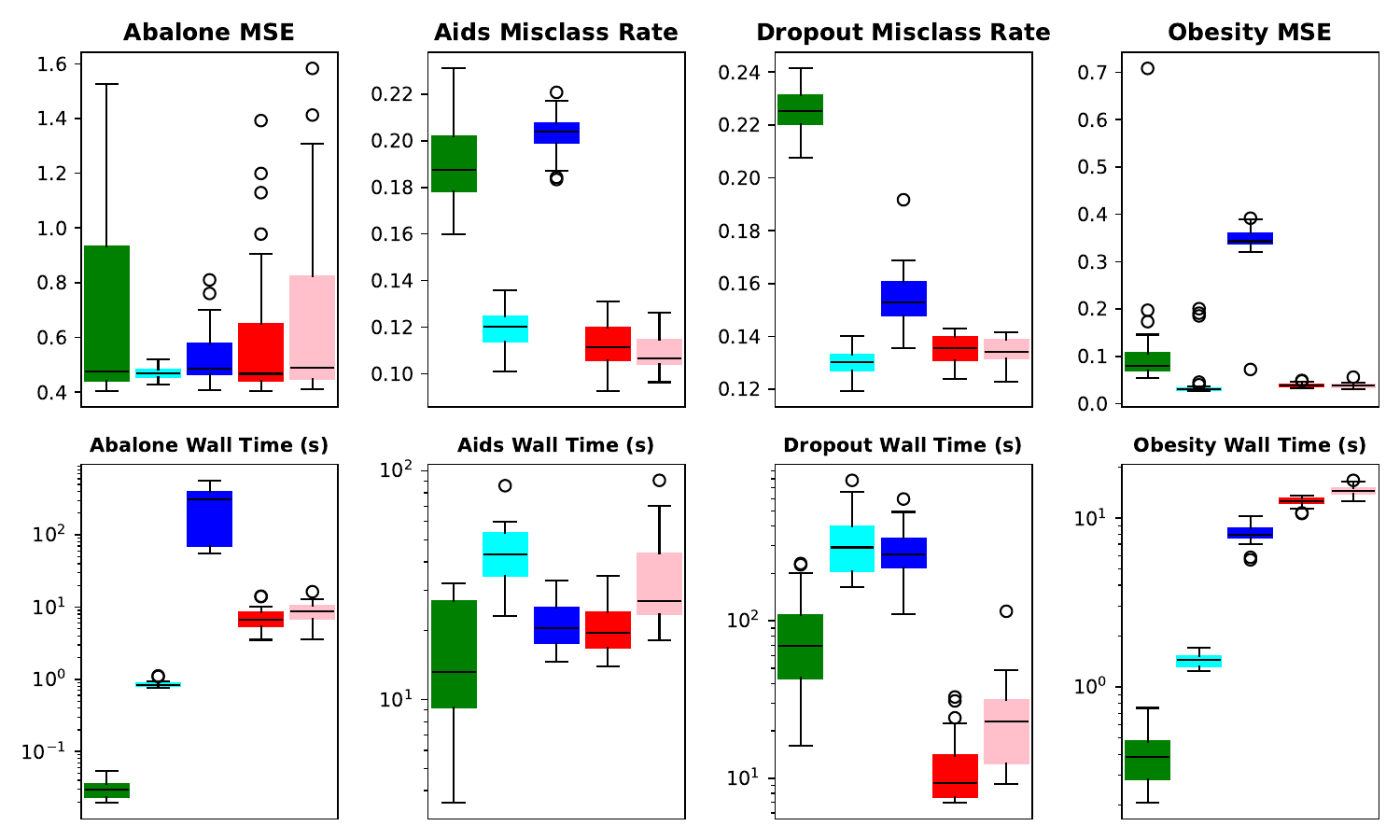}
    \includegraphics[width=0.9\textwidth]{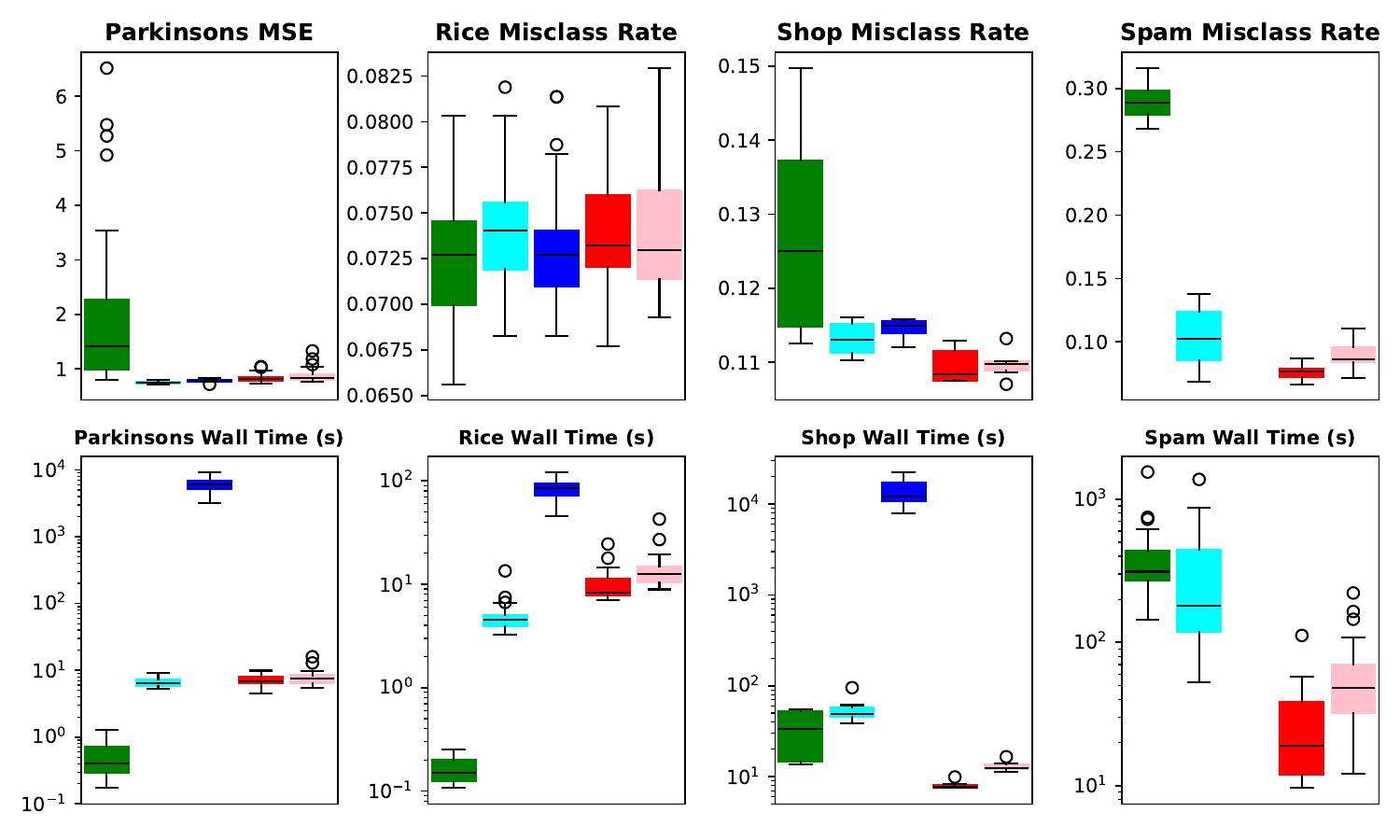}
    \caption{Comparison on real datasets in terms of prediction accuracy. }
    \label{fig:libsvm}
\end{figure}

\subsection{Hierarchical Second Order Regression on Real Data Study Results}\label{sec:exp_libsvm}

Figure \ref{fig:libsvm} shows the results of the real data study.
We find that the proposed methodology is consistently among the best predictors.
Compared to the other structured method \texttt{MLGL}, we see that the proposed estimator achieves orders of magnitude faster execution time on the larger test problems (Parkinsons, Spam, Dropout, Shop) while offering improved accuracy on 5/8 problems, and 3/4 of the biggest datasets.
Furthermore, it is not consistent whether grouped or independent sparsity assumptions lead to improved predictions; 
it may be that the venerated effect hierarchy principle (see e.g. \citet[Chapter 1]{wu2011experiments}) is not particularly relevant to these datasets.

\subsection{Deep Active Subspaces for Vaccination Behavior Results}\label{sec:mosaic}

\begin{figure}
    \centering
    \newcommand{\smg}{0.65}

    \includegraphics[width=0.35\textwidth]{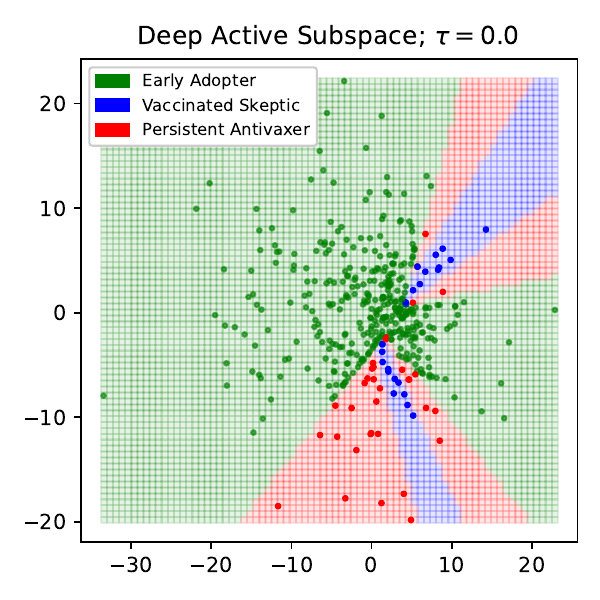}
    \includegraphics[width=0.35\textwidth]{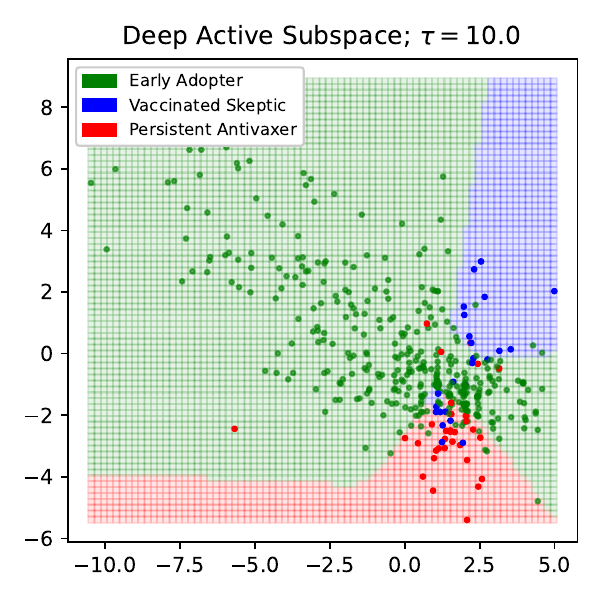}
    \includegraphics[width=0.22\textwidth]{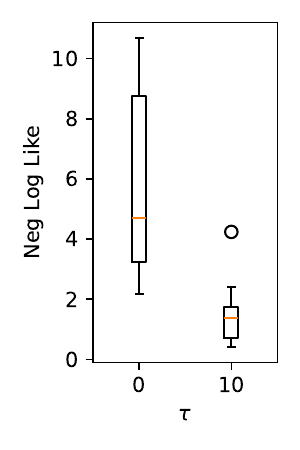}
    

    %
    %
    
    \caption{Regularized Deep Active Subspace reveals two vaccination hesitant groups. 
    Plots show the active subspace projection of the training data
    \textit{Left}: without penalty
    \textit{Mid}: with a penalty strength of $\tau=10$.
    \textit{Right}: 10 fold Cross Validation reveals that the penalty improves predictions.
    }
    \label{fig:mosaic}
\end{figure}

%

Figure \ref{fig:mosaic} shows the results of the DASC analysis for two different regularization levels, first the unpenalized $\tau=0$ and then with a penalty strength $\tau=10$.
We see that the unpenalized deep active subspace produces six major non-contiguous regions for the vaccinated skeptics and persistent antivaxers, while the penalized active subspace yields just a single group containing predominantly persistent antivaxers and a second containing predominantly vaccinated skeptics. 
Furthermore, 10-fold cross validation reveals that this representation outperforms in terms of predictive negative log-likelihood (Figure \ref{fig:mosaic} bottom left).

Intriguingly, rather than finding a continuous spectrum from early adopter to vaccinated skeptics to persistent antivaxers, the regularized analysis projects the vaccinated skeptics and persistent antivaxers such that they are separated by the early adopters.
Analysis of the accounts followed by respondents higher on the y-axis showed that they appear to be comparatively politically disengaged (see Appendix \ref{sec:app_mosaic}), and the majority will eventually be vaccinated.
On the other hand, the early adopters and persistent antivaxers were more likely to be politically engaged, following Democratic and Republican accounts, respectively.


\subsection{Sparse Hurdle Gravity Model for International Migration Results}\label{sec:hurdle}

Figure \ref{fig:hcr} shows parameter estimates for the international migration case study as a function of penalty strength.
We see that the first three parameters to enter the model are interaction terms for $\b_\mu$ involving the \texttt{contig} variable, which is an indicator encoding whether two countries share a border. 
The other variables involved in the interaction are \texttt{colony}, which is an indicator variable encoding colonial ties between the countries, \texttt{deaths\_o} which is the number of conflict deaths in the origin country on the log scale, and \texttt{dist} which gives the great circle distance between the capitals of the two countries. 
Consequently, the model is predicting migration between neighboring countries with colonial ties from a country with many deaths towards countries whose capitals are near.
The right panel shows the predictive log-likelihood.
We see that this procedure was able to comb through about $200,000$ possible parameters to find a simple four-variable model that achieves predictive performance superior to the full and null models.


From a substantive perspective, these results may be viewed as validating the concept behind the ``gravity model” of international migration which centers spatiality as the primary driver of human migration dynamics.
It is notable that economic variables, such as GDP and inflation, were not included in this optimal model despite their being available.
Since we are here looking only at claims for refugee status, which according to the 1951 Refugee Convention extends only to individuals who fear persecution if they were returned home (based on one of five reasons) and not to persons seeking to improve economic opportunities, this may seem \textit{a priori} reasonable.
But our study takes place amidst a contemporary legal-political context in which many argue the legal protections of refugees should stretch to accommodate “survival migrants” who are both unable to return home due to fear of persecution and experience economic deprivation; see, e.g. Betts (2013).
The lack of inclusion of any economic variables is therefore notable.


\begin{figure}
    \centering
    \includegraphics[width=\textwidth,trim={0.9cm 0.55cm 0 0},clip]{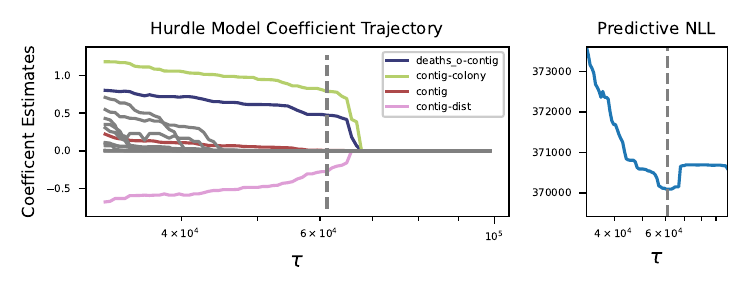}
    \caption{International migration case study.
    \textit{Left:} Regression coefficient trajectory versus regularization strength $\tau$.
    \textit{Right:} Predictive NLL versus $\tau$; lower is better.}
    \label{fig:hcr}
\end{figure}


\section{Discussion}\label{sec:conclusion}

We first summarize our developments and findings before offering substantive conclusions and proposing directions for future work.


\subsection{Summary}
This article developed methods to optimize loss functions associated with the adaptive Lasso penalty with any smooth likelihood and sparsity structure by treating the penalty coefficients as just another decision variable.
We first studied the joint optimization problem with respect to the model parameters $\bb$ and the penalty coefficients $\bl$, developing proximal operators therefor. 
We next developed the basic asymptotic properties of this procedure.
After, we demonstrated how this scheme could induce structured sparsity by proposing some joint hyperpriors for the penalty coefficients.
Next, we pivoted to numerical experiments, first confirming the advantage of the method when the model is correctly specified in simulation studies before finding favorable comparisons to state of the art methods on a battery of real datasets using various likelihoods.
Finally came our main applications.
First, a sparse neural network analysis of vaccination behavior which was able to create a better projection when forced to be sparse. 
And second, a model of international migration using a hurdle model with group sparsity on a large dataset which identified four parameters of around $200,000$ that led to a good predictive model.

\subsection{Conclusions}
We advocate the proposed methodology as an off-the-shelf strategy which can be the first deployed in a black-box manner when working with a complicated likelihood and sparsity structure and a huge dataset, particularly when it is not desired or not possible to investigate the specific structure of that likelihood-sparsity combination.
We hope that it will allow practitioners to use likelihoods and sparsity structures which best fit a particular dataset rather than being restricted to those structures with already well understood proximal operators.
This approach allows bias-free structured sparsity imposition in the general setting, and is easily incorporated into an existing workflow implemented via automatic-differentiation and linear algebra software, which we demonstrated using JAX.

Substantively, with the help of our sparsity procedure we found interesting conclusions for our case studies. 
In the vaccination behavior study, we found that early adopters of the COVID-19 vaccine and those who did not obtain it during the case study were neighbors in the projection implied by the neural network, with those initially hesitant to vaccinate but eventually willing to do so on the outside. 
This was because these first two groups resembled one another in being politically engaged. 
In our international refugee claims case study, the sparsity structure found a simple model which emphasized spatial relations and conflict over economic or political terms which performed well from an out-of-sample perspective.


\subsection{Future Work}
Conducting this research brought into sharp relief challenges associated with performing sparse analysis in nonlinear models in general.
In particular, we found that optimizations on sequences of penalty strengths resulted in noisy coefficient estimates that benefited from postprocessing by a rolling median.
Additionally, we found it advantageous to reset the hyperparameters associated with $\bl$ in the hierarchical priors on each iterate, which disrupts warm-starting.
Also, we did not find a way to make Nesterov acceleration profitable in this context, as often it would propose negative values for $\l$ \textit{after} the proximal step, which violates the positivity constraint.
Finally, while gradient-based methods are quite scalable in terms of per-iteration complexity to massive problems, when the conditioning is poor, the total number of iterations required may be prohibitive. 
This is the case, for example, for a regression with significant multicollinearity in the predictors.
In this setting, higher order optimization methods, namely Fisher Scoring, can prove to be more efficient.
However, from the perspective of the proximal gradient algorithm, this would correspond to a non-diagonal preconditioner, which couples the proximal problem and would require special consideration for efficient solution.
Future work resolving these difficulties has the potential to even further extend the reach of the proposed approach and sparse nonlinear models in general.

There are also open threads within our particular framework.
Firstly, there are many possibilities for sparsity structures left to be explored, such as spatiotemporal sparsity or that based on graph structures.
Since any sparsity structure can benefit from the approach we have outlined, this work is mainly empirical, and would be to determine which sparsity structures most benefit from the either the debiasing effect that the variable-penalty framework has, or from the possibility of accelerated computation through the deployment of a simple proximal operator.
Second, the development of the proximal operator for the variable-penalty $\ell_1$ norm opens the door to a new class of proximal operators associated with other penalties if we allow their penalty coefficients to vary. 
Here, there is interesting work to be done in terms of either finding closed form representations or efficient iterative algorithms for the action of these proximal operators, as well as the empirical question of in what situations this would be superior to the one-size-fits-all approach we have described here.
Additionally, it may be worth investigating the possibility of some convex-analytic structure associated with this new class of proximal operators.

\acks{The authors gratefully acknowledge funding from the Massive Data Institute, the Tech and Public Policy Grant Program, and NSF/NGA grant \#2428033, as well as data from the United Nations High Commissioner for Refugees.}

\bibliography{sample}

\appendix

\section{Extended Proofs for Section 3}\label{sec:app_proofs}

For maximum accessibility, we rely as much as possible on elementary calculus and use subdifferential calculus sparingly.

\begin{customlemma}{1}
    The marginal cost of P1 with respect to $\lambda$ (i.e. with $x$ profiled out) is the following piecewise quadratic expression:
\begin{equation}
     \underset{\lambda>0}{\mathrm{argmin}}  \begin{cases}
        \frac{1}{2}(\frac{1}{s_\lambda} - s_x)\lambda^2+(|x_0|-\frac{\lambda_0}{s_\lambda})\lambda + \frac{\lambda_0^2}{2s_\lambda} & \lambda < \frac{|x_0|}{s_x} \\
        \frac{(\lambda-\lambda_0)^2}{2s_\lambda}+\frac{x_0^2}{2s_x} & \lambda \geq \frac{|x_0|}{s_x}\,\, , \\
     \end{cases}
\end{equation}
where the changepoint $\lambda=\frac{|x_0|}{s_x}$ is the point where $\lambda$ is just large enough to push $x$ to zero.
\end{customlemma}
\begin{proof}
    Convert to nested optimization and exploit the fact that the solution for known $\lambda$ is given by the soft thresholding operator:
\begin{align}
    &\underset{x\in\mathbb{R},\lambda>0}{\mathrm{argmin}} \,\,  \lambda |x|+\frac{(x-x_0)^2}{2s_x} + \frac{(\lambda-\lambda_0)^2}{2s_\lambda} \\
    &\iff 
    \underset{\lambda>0}{\mathrm{argmin}} \,\, \frac{(\lambda-\lambda_0)^2}{2s_\lambda} + \underset{x\in\mathbb{R}}{\mathrm{argmin}} \,\, \lambda |x|+\frac{(x-x_0)^2}{2s_x}  \\
    &\iff \underset{\lambda>0}{\mathrm{argmin}} \,\, \frac{(\lambda-\lambda_0)^2}{2s_\lambda} + \lambda (|x_0|-s_x\lambda)^+ +\frac{(|x_0|-s_x\lambda)^{+,2}-2|x_0|(|x_0|-s_x\lambda)^+}{2s_x}  \,\, . \label{eq:simpl}
\end{align}
\end{proof}

\begin{customthm}{2}
    The optimizing $\lambda$ for the proximal program P1 is given by, when $s_xs_\lambda<1$:
    \begin{equation}\label{eq:prox1s}
        \lambda^* =\begin{cases} 
          \lambda_0 & \lambda_0 \geq \frac{|x_0|}{s_x} \\
          \frac{(\lambda_0-s_\lambda|x_0|)^+}{1-s_\lambda s_x} & o.w. \,\,\,\, ,
       \end{cases} 
    \end{equation}
    and by $\lambda^* = \mathbbm{1}_{\big[\frac{\lambda_0}{\sqrt{s_l}} > \frac{|x_0|}{\sqrt{s_x}}\big]} \lambda_0$ (here $\mathbbm{1}$ denotes the indicator function) otherwise.
In either case $x^* = (|x_0|-s_x\lambda^*)^+\mathrm{sgn}(x_0)$.
\end{customthm}
\begin{proof}
    We need only find the optimum of each interval. When $\lambda\geq \frac{|x_0|}{s_x}$, the optimum is simply as close as we can get to $\lambda_0$, namely $\lambda\gets\max[\lambda_0,\frac{x_0}{s_x}]$. 
On the other hand, when $\lambda \leq \frac{|x_0|}{s_x}$, if $s_xs_\lambda<1$, the optimum is as close as we can get to the stationary point $\frac{(\lambda_0-s_\lambda|x_0|)}{1-s_\lambda s_x}$, explicitly $\lambda\gets\min[\frac{(\lambda_0-s_\lambda|x_0|)^+}{1-s_\lambda s_x},\frac{|x_0|}{s_x}]$. When $s_xs_\lambda\geq1$, however, the solution is at one of the interval boundaries $[0,\frac{|x_0|}{s_x}]$; the boundaries have costs of $\frac{\lambda_0^2}{2s_\lambda}$ and $\frac{(\frac{|x_0|}{s_x}-\lambda_0)^2}{2s_\lambda}+\frac{x_0^2}{2s_x}$, respectively, and so we choose $\lambda\gets 0$ if $\frac{\lambda_0^2}{2s_\lambda}<\frac{(\frac{|x_0|}{s_x}-\lambda_0)^2}{2s_\lambda}+\frac{x_0^2}{2s_x}$ and $\lambda\gets\frac{|x_0|}{s_x}$ otherwise. But the cost at $\lambda=\lambda_0$ is only $\frac{x_0^2}{2s_x}$, so the choice is between $0$ and $\lambda_0$ with costs $\frac{\lambda_0^2}{2s_\lambda}$ and $\frac{x_0^2}{2s_x}$. 
\end{proof}

\begin{customthm}{5}
    Consider the problem
    $\underset{x\in\mathbb{R},\lambda>0}{\mathrm{argmin}} \,\,  \lambda |x|+\frac{(x-x_0)^2}{2s_x} + \frac{(\lambda-\lambda_0)^2}{2s_\lambda} - a \log\lambda$.
    If $2\beta_0 \leq s_\beta (\lambda_0^2 + \sqrt{\lambda_0^2+4 s_\lambda'})$, the solution will be:
    \begin{equation}
        (\beta^*,\lambda^*) = 
        \Bigg(0,
        \frac{\lambda_0^2 + \sqrt{\lambda_0^2+4 s_\lambda a}}{2}
        \Bigg)\,.
    \end{equation}
    Otherwise, we have
    \begin{equation}
        \lambda^* = \frac{
        s_\lambda|\beta_0| - \lambda_0
        \pm\sqrt{
        (\lambda_0-s_\lambda\beta_0)^2 + 4(1-s_\beta s_\lambda) s_\lambda a
        }
        }{
        2(s_\beta s_\lambda-1)
        }\,
    \end{equation}
    and subsequently 
    $\beta^* = (|\beta_0|-s_\beta\lambda^*)^+\mathrm{sgn}(\beta_0)$.
\end{customthm}

\begin{proof}

The stationarity conditions give:
\begin{align}
    & 0 \in \lambda \textrm{ sgn}(\beta) + \frac{\beta-\beta_0}{s_\beta} \tag{a}\\
    & 0 = |\beta| + \frac{\lambda-\lambda_0}{s_\lambda} - \frac{a}{\lambda} \tag{b} \, .
\end{align}

Let's first consider these conditions when $\beta=0$.
In this case, (b) gives us that $\lambda = \frac{\lambda_0^2 \pm \sqrt{\lambda_0^2+4 s_\lambda a}}{2}$, and since the smaller of these is always negative, we have $\lambda^* =  \frac{\lambda_0^2 + \sqrt{\lambda_0^2+4 s_\lambda a}}{2}$.

Plugging this into (a) yields that $|\beta_0|\leq s_\beta\lambda^*$.
So we have that $\beta^*=0$ when $2\beta_0 \leq s_\beta (\lambda_0^2 + \sqrt{\lambda_0^2+4 s_\lambda a})$.

If not, we know that $\beta^*\neq 0$ and (a) becomes an ordinary equation. 
In this case, we get that
\begin{equation}
    \lambda^* = \frac{
    s_\lambda|\beta_0| - \lambda_0
    \pm\sqrt{
    (\lambda_0-s_\lambda\beta_0)^2 + 4(1-s_\beta s_\lambda) s_\lambda a
    }
    }{
    2(s_\beta s_\lambda-1)
    } \, .
\end{equation}
\end{proof}

\section{Extended Proofs for Section 4}\label{sec:asymp_appendix}

\begin{customlemma}{7}
	The following hold, where $\lambda^*$ denotes the optimizing $\lambda$, and is formally a function of $\tau$ and $|\beta|$:
 \begin{multicols}{2}
	\begin{enumerate}
		\item $\lambda^* = \frac{1}{\tau|\beta| + \rho'(\lambda^*)}$.
		\item $\frac{\partial \lambda^*}{\partial |\beta|} = -\frac{\tau}{\frac{1}{\lambda^{*2}} + \rho''(\lambda^*)}$.
		\item $g_{\tau}'(|\beta|) =  \tau \lambda^*$.
		\item $g_{\tau}''(|\beta|) =  -\frac{\tau_n^2}{(\tau_n+\rho'(\lambda^*))^2+\rho''(\lambda^*)}$.
	\end{enumerate}
 \end{multicols}
\end{customlemma}
\begin{proof}
    It will be convenient to develop notation for the cost function inside our penalty: $g_\tau(|\beta|) = \underset{\lambda>0}{\min} \,\, \big[\tau\lambda|\beta| - \log\lambda + \rho(\lambda)\big] := \underset{\lambda>0}{\min} \,\, c^p(|\beta|,\lambda)$.
    For 1, since $\lambda^*$ is the optimizing $\lambda$, and due to the $-\log\lambda$ constraining the optimum to be an interior point, we know that $0=\frac{\partial}{\partial\lambda} \big[\tau\lambda|\beta| - \log\lambda + \rho(\lambda)\big] = \tau|\beta| - \frac{1}{\lambda} + \rho'(\lambda)$.
    For 2, we can use implicit differentiation on this same equation.
    For 3, we simply note that $\frac{\partial}{\partial|\beta|}g_{\tau}(|\beta|) = \frac{\partial}{\partial|\beta|} \big[\tau\lambda^*|\beta| - \log\lambda^* + \rho(\lambda^*)\big] = \tau\lambda^* + \frac{\partial\lambda^*}{\partial|\beta|}\frac{\partial c^p(|\beta|,\lambda)}{\partial\lambda}\Bigr|_{\lambda^*} = \tau\lambda^*$. 
    4 proceeds by differentiating 3 and plugging in 1 and 2.
\end{proof}

\begin{customthm}{8}
    Assume that the logarithmic derivative of the hyperprior density on $\lambda$ is bounded ($|\rho'(\lambda)|\leq M_1 \,\, \forall \lambda\geq0$) and that the density is decreasing on $(0,\infty)$. Then:
    \begin{enumerate}
        \item $g'_{\tau}(|\beta|)\approx\frac{1}{|\beta|}$ for large $\beta$.
        \item The minimum of $|\beta|+g'_{\tau}(|\beta|)$ is achieved at $\beta=0$ with value $\lambda_a\tau$.
    \end{enumerate}
\end{customthm}
\begin{proof}
    For 1, $g'(|\beta|) = \frac{\tau}{\tau|\beta|+\rho(\lambda^*)}$, and since $\rho(\lambda)$ is bounded, $\underset{|\beta|\to\infty}{\lim}\frac{\tau}{\tau|\beta|+\rho(\lambda^*)} = \frac{1}{|\beta|}$.
    For 2, let $\lambda_a$ be the $\lambda$ such that $\frac{1}{\lambda_a} = \rho'(\lambda_a)$ (which is unique by the assumption that $\rho$ is increasing). Note that $\lambda_a \leq \lambda^*$ and $\lambda_a = \lambda^*(0)$. So each term of $|\beta|+\lambda^*\tau$ is decreasing in $|\beta|$ individually, and so the minimum of their sum must occur at $0$, yielding value $\lambda_a\tau$.
\end{proof}

\begin{customthm}{9}
    Let $\tau_n=n\tau_0$ for $\tau_0>0$, and further assume that $|\rho''(|\lambda|)|<M_2$ (bounded second logarithmic derivative). Then, under the following standard regularity conditions on the likelihood:
    \begin{enumerate}
        \item The data $\mathbf{y}_i$ are i.i.d. with density function $f(\y;\bb)$ providing for common support and model identifiability. 
        We assume it has a score function with expectation zero $\mathbb{E}_{\bb}\Big[\nabla_{\bb} \log f(\y;\bb)\Big] = \mathbf{0}$ and a Fisher information expressible in terms of second derivatives: $I(\bb) = \mathbb{E}_{\bb}\Big[\nabla_{\bb}^2 \log f(\y;\bb) \Big]$. 
        \item The information matrix is finite and positive definite when $\bb=\tb$, with $\tb$ the true parameter vector. 
        \item For some open subset $\mathcal{B}$ containing $\tb$,  for almost all $\y$, the density is thrice differentiable $\forall\bb\in\mathcal{B}$ and that $\Bigr|\frac{\partial^3 \log f(\y;\bb)}{\partial\beta_i\partial\beta_j\partial\beta_k}\Bigr|\leq M_{i,j,k}(\y)$, also over $\mathcal{B}$, where the functions $M$ are such that $\mathbb{E}_{\tb}[M_{i,j,k}(\y)]<\infty$.
    \end{enumerate} 
   there is a local minimum of $Q(\bb) = -l(\bb) + \sum_{p=1}^P g_{\tau_n}(|\beta_p|)$ that satisfies the following:
    \begin{enumerate}
        \item $\hat{\bb}_2=\mathbf{0}$ with probability approaching 1 as $n\to\infty$.
        \item $\hat{\bb}_1$ is asymptotically normal with covariance given approximately by $\frac{1}{n}I(\tb_1)$, the Fisher information matrix considering only active variables.
    \end{enumerate}
\end{customthm}
\begin{proof}

This proof is based on the derivation in \citet{fan2001variable} Theorem 2.
However, there are some differences between their context and ours.
Most saliently, while the penalty that \citet{fan2001variable} study grows so as to overpower the likelihood within a neighborhood of size $O_p(\frac{1}{\sqrt{N}})$ of the origin, the penalty studied here achieves this only within a neighborhood of size $O_p(\frac{1}{N})$ of the origin, requiring a slight variation of the proof strategy.

We assume for convenience that it is the first $R$ entries of $\b$ that are nonzero, i.e. $\tb^\top = [\tb_1, \tb_2]^\top=[\tb_1, \mathbf{0}]^\top$.
We'll denote its estimator as $\hb^\top = [\tb_1, \tb_2]^\top$.

We will begin by assuming that $\hb_2$ is clamped to $\mathbf{0}$, and seek to show that in such a situation $\hb_1$ converges to $\tb_1$ with the usual rate.
We will start by establishing that, with any probability $1-\epsilon$ and $N$ sufficiently large, we have
\begin{equation}
     \underset{\Vert\uu\Vert_2 = C}{\min}
     Q\Bigg(
     \begin{bmatrix}
         \tb_1 \\ \mathbf{0}
     \end{bmatrix}
     + \frac{1}{\sqrt N}
     \begin{bmatrix}
         \uu \\ \mathbf{0}
     \end{bmatrix}
     \Bigg) > Q(\tb)    
\end{equation}
This is because, using a series expansion of the log likelihood about $\tb$, we can write:
\begin{align*}
     &
     Q\Bigg(
     \begin{bmatrix}
         \tb_1 \\ \mathbf{0}
     \end{bmatrix}
     + \frac{1}{\sqrt N}
     \begin{bmatrix}
         \uu \\ \mathbf{0}
     \end{bmatrix}
     \Bigg) - Q(\tb)    
     =
     - 
     \frac{\nabla_{\b_1} l(\tb)^\top}{\sqrt N} \uu + 
     \frac{1}{2} \uu^\top I(\tb_1) \uu (1+o_P(1)) \\
     & 
     + \sum_{p=1}^R \frac{1}{\tilde{\beta}_p} \frac{u_p}{\sqrt N}\big(1+o(1)\big)
     + \sum_{p=1}^R \frac{u_p^2}{N}\big(1+o(1)\big) \,,
\end{align*}
where we have used the asymptotics for $g$ implied by Lemma 7 and used the fact that $-\frac{\nabla^2_{\b_1} l(\tb)}{N}=I(\tb_1)+o_P(1)$ by assumption.

Since $\frac{\nabla_{\b_1} l(\tb)^\top}{\sqrt N}$ is stochastically bounded, it's clear that it will eventually be dominated by the term quadratic in $\uu$ for sufficiently large $C$. 
Since the terms introduced by the penalty vanish, the quadratic term will dominate them too, and since the matrix $I(\tb_1)$ is positive definite (as it is a submatrix of $I(\tb)$ which was assumed positive definite), the whole expression is negative.


Next, we will show that this $\hb^\top=[\hb_1, \mathbf{0}]^\top$ is actually optimal with respect to $\hb_2$ as well with high probability for sufficiently large $N$.
We'll do this by showing that, for $C>0$:
\begin{equation}
    \textrm{sgn}\Bigg(\frac{\partial Q(\hb)}{\partial \beta_p}\Bigg)
    =
    \textrm{sgn}(\beta_p)
    \,\, \textrm{ for } \beta_p \in \Big(-\frac{C}{N},0\Big)\cup \Big(0,\frac{C}{N}\Big) \,\,. 
\end{equation}

Starting again with a series expansion yields:
\begin{align}
    &
    \frac{\partial Q(\hb)}{\partial \beta_p}
    =
    - \frac{\partial l(\tb)}{\partial \beta_p}
    - \sum_{i=1}^R \frac{\partial^2 l(\tb)}{\partial \beta_p \partial \hat\beta_i} 
    (\hat{\beta_i}-\tilde{\beta_i})
    \\&
    - \sum_{i=1}^R\sum_{j=1}^R \frac{\partial^3 l(\b^*)}{\partial \beta_p \partial \beta_i \partial \beta_j} 
    (\hat{\beta_i}-\tilde{\beta_i})
    (\hat{\beta_j}-\tilde{\beta_i})
    + \frac{1}{\beta_p}(1+o(1))
\end{align}
where $\b^*$ is associated with the Peano form of the remainder.
The dominant term in the likelihood is 
$- \frac{\partial l(\tb)}{\partial \beta_p}$, which is $O_p(\sqrt{n})$.
By contrast, $|\frac{1}{\beta_p}|>\frac{N}{C}$.
So the penalty's contribution dominates all other terms in magnitude, and the sign of the expression is the sign of $\frac{1}{\beta_j}$.

This establishes that there is a local optimizer of $Q$ which has both $\hb_2=0$ and $\Vert\hb_1-\tb_1\Vert_2 = O_P(\frac{1}{\sqrt{n}})$.
We'll conclude by establishing the asymptotic distribution of $\hb_1$.
Since 
\begin{align}
    \mathbf{0} = \frac{1}{\sqrt N}\nabla_{\b_1} Q(\b)\big\vert_{\hb_1}
    =
    - \frac{1}{\sqrt{N}}\nabla_{\b_1} l(\tb) 
    - \sqrt{N} I(\tb_0) (\hb-\tb) + o_P(1) - 
    \frac{1}{\sqrt N}
    \begin{bmatrix}
        \frac{1}{\hat\beta_1} \\
        \vdots \\
        \frac{1}{\hat\beta_R}
    \end{bmatrix}
\end{align}
such that by the fact that 
$
\frac{1}{\sqrt{N}}\nabla_{\b_1} l(\tb) \to N(\mathbf{0},\mathbf{I})
$
together with Slutsky's theorem:
\begin{equation}
    \sqrt{N} (\hb_1 - \tb_1) \to N\big(\mathbf{0}, I(\tb_1)^{-1}\big) \,\,.
\end{equation}

\end{proof}

\section{Vaccination Behavior Case Study Additional Analysis}\label{sec:app_mosaic}

\begin{figure}
    \centering
    \newcommand{\smg}{0.4}
    \includegraphics[scale=\smg]{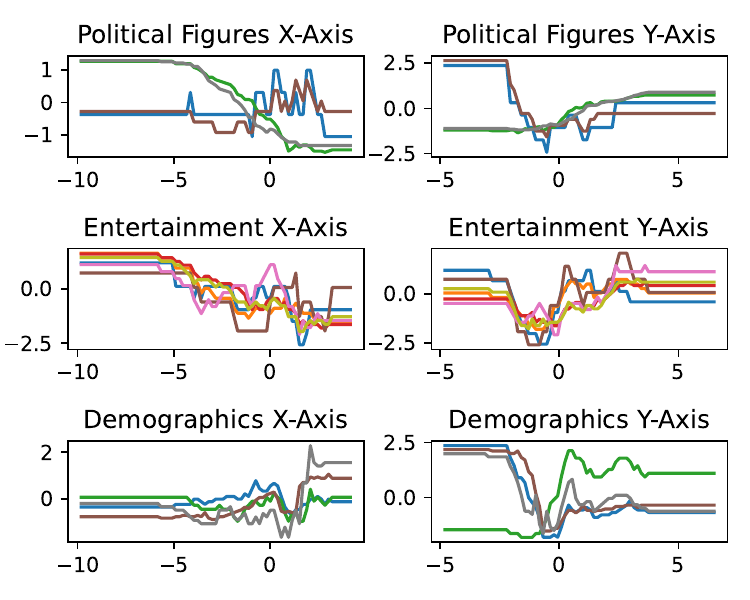}
    \includegraphics[scale=\smg]{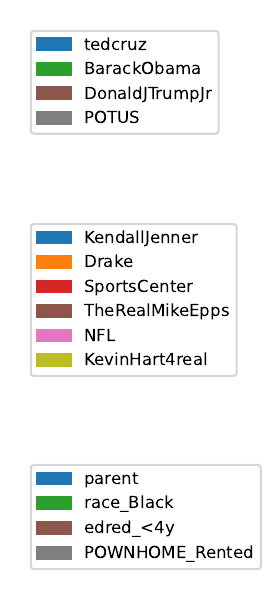}
    \includegraphics[scale=\smg]{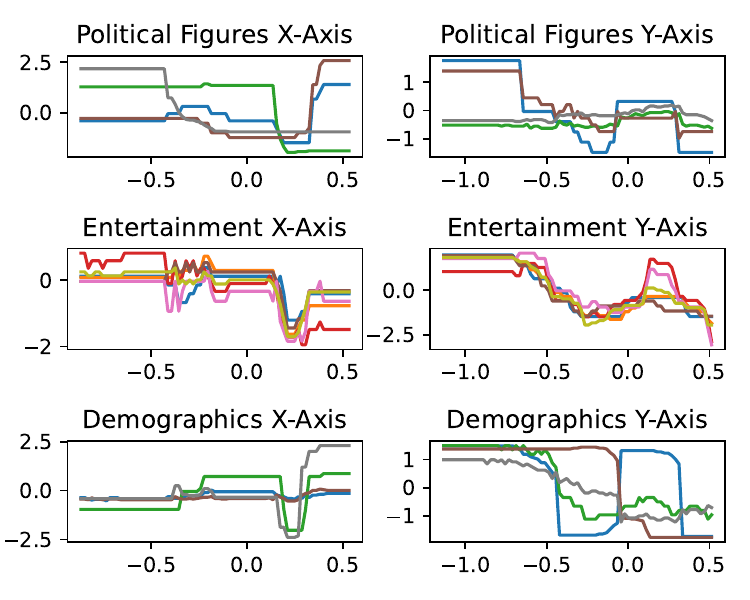}
    \includegraphics[scale=\smg]{images/mosaic_nn_legend.pdf}
    \caption{Inverse Regressions for the high penalty (left) and low penalty (right) settings.}
    \label{fig:inverse}
\end{figure}

In order to gain intuition about the learned low-dimensional structure, we perform inverse regression \citep{li1991sliced}; see Figure \ref{fig:inverse}.
The value of an input variable along a reduced axis is given by the average of the 100 points which are closest to a given tick of that axis after projection.
We show only those variables that were still selected in higher-penalty analysis with $\tau = 100$ (not pictured).
We see that low values of x-axis encodes engagement with the Democratic party, following Barack Obama and the POTUS account (Joe Biden when data were recorded), while high values of the x-axis are associated with renting and lower engagement with entertainment figures.
Low values of the y-axis are associated with following Republican accounts (Ted Cruz and Donald Trump Jr) as well as with renting, parenthood, and lack of a 4 year college degree.
High y-axis values, on the other hand, are associated with less political engagement and the Black demographic variable.
The middle of the y-axis is less likely to follow entertainment accounts than the extremes.

\end{document}